\documentclass[10pt,twocolumn,letterpaper]{article}

\usepackage{iccv}              % To produce the CAMERA-READY version
\usepackage{amsmath}

\DeclareMathOperator*{\argmin}{arg\,min}
\usepackage{amssymb}
\usepackage{mathtools}
\usepackage{amsthm}

\usepackage{microtype}
\usepackage{graphicx}
\usepackage{subcaption}

\usepackage{comment}
\usepackage{tikz}
\usetikzlibrary{positioning, calc}
\usetikzlibrary{tikzmark}
\usepackage{algorithm}
\usepackage{algorithmic}
\usepackage{float}
\usepackage{csquotes}
\usepackage{wrapfig}
\usepackage{multirow}
\usepackage{makecell}
\usepackage{pgfplots}
\usepackage{asymptote} % Load Asymptote package
\usepackage{booktabs}

\theoremstyle{plain}
\newtheorem{theorem}{Theorem}[section]
\newtheorem{proposition}{Proposition}
\newtheorem{lemma}{Lemma}

\theoremstyle{definition}
\newtheorem{definition}[theorem]{Definition}

\theoremstyle{remark}

\usepackage[textsize=tiny]{todonotes}

\newcommand{\bzero}{\mathbf{0}}

\newcommand{\bb}{\mathbf{b}}

\newcommand{\bx}{\mathbf{x}}

\newcommand{\bbh}{\hat{\mathbf{b}}}

\newcommand{\alice}[1]{\prescript{\mathcal{A}}{}{#1}}
\newcommand{\bob}[1]{\prescript{\mathcal{B}}{}{#1}}
\newcommand{\cover}[1]{\prescript{\mathcal{C}}{}{#1}}
\definecolor{iccvblue}{rgb}{0.21,0.49,0.74}
\usepackage[pagebackref,breaklinks,colorlinks,allcolors=iccvblue]{hyperref}
\usepackage[capitalize,noabbrev]{cleveref}

\def\paperID{12041} % *** Enter the Paper ID here
\def\confName{ICCV}
\def\confYear{2025}

\title{PSyDUCK: Training-Free Steganography for Latent Diffusion}

\author{Aqib Mahfuz$^*$\\
University of Oxford\\
{\tt\small aqib.mahfuz@gmail.com}
\and
Georgia Channing$^*$\\
University of Oxford\\
{\tt\small cgeorgia@robots.ox.ac.uk}
\and
Mark van der Wilk\\
University of Oxford\\
{\tt\small mark.vdwilk@cs.ox.ac.uk}
\and
Philip Torr\\
University of Oxford\\
{\tt\small philip.torr@eng.ox.ac.uk}
\and
Fabio Pizzati\\
University of Oxford\\
{\tt\small fabio.pizzati@eng.ox.ac.uk}
\and
Christian Schroeder de Witt\\
University of Oxford\\
{\tt\small cs@robots.ox.ac.uk}
}

\begin{document}
\maketitle

\begin{abstract}

Recent advances in generative AI have opened promising avenues for steganography, which can securely protect sensitive information for individuals operating in hostile environments, such as journalists, activists, and whistleblowers. However, existing methods for generative steganography have significant limitations, particularly in scalability and their dependence on retraining diffusion models. We introduce PSyDUCK, a training-free, model-agnostic steganography framework specifically designed for latent diffusion models. PSyDUCK leverages controlled divergence and local mixing within the latent denoising process, enabling high-capacity, secure message embedding without compromising visual fidelity. Our method dynamically adapts embedding strength to balance accuracy and detectability, significantly improving upon existing pixel-space approaches. Crucially, PSyDUCK extends generative steganography to latent-space video diffusion models, surpassing previous methods in both encoding capacity and robustness. Extensive experiments demonstrate PSyDUCK’s superiority over state-of-the-art techniques, achieving higher transmission accuracy and lower detectability rates across diverse image and video datasets. By overcoming the key challenges associated with latent diffusion model architectures, PSyDUCK sets a new standard for generative steganography, paving the way for scalable, real-world steganographic applications.
\end{abstract}

\section{Introduction}
\label{introduction}

\begin{figure}[t]
    \centering
    \includegraphics[width=\linewidth]{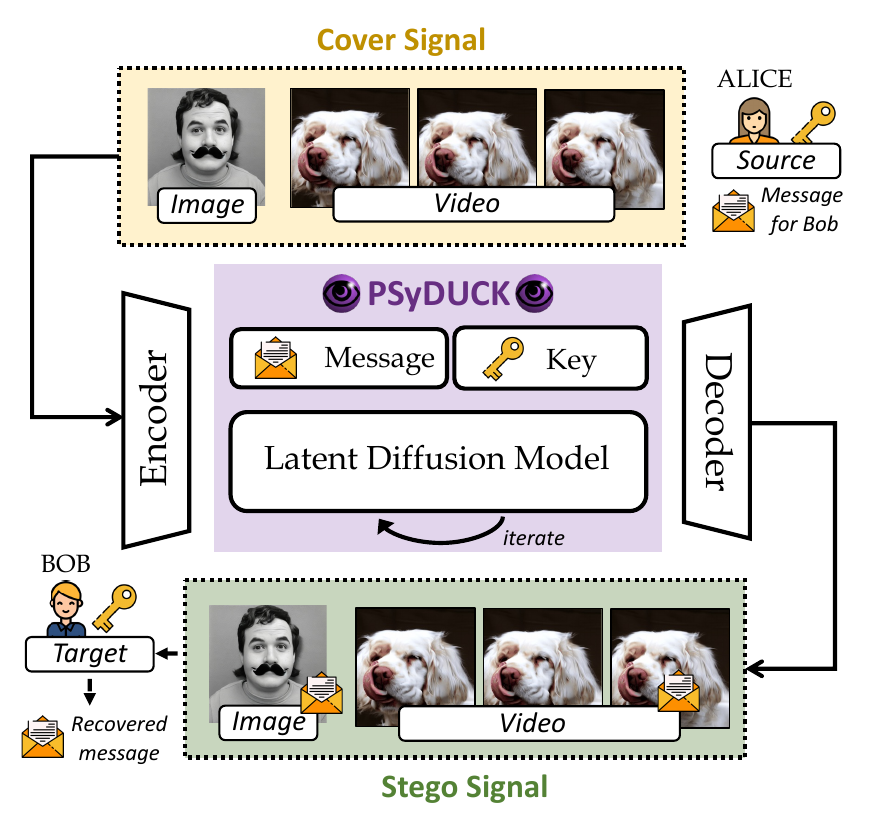}
    \caption{\textbf{General PSyDUCK scheme.} Alice wants to send a secret message to Bob while preventing interception by malicious actors. To achieve this, PSyDUCK embeds a steganographic message of arbitrary length into a \textit{cover signal}--an image or video--using shared keys and pre-trained latent diffusion models. The resulting \textit{stego-signal} can be freely shared on the open web, allowing Bob, who possesses the correct key, to decode and retrieve the original message.}
    \label{fig:teaser}
\end{figure}

Steganography, the practice of hiding information within ordinary media, has long protected sensitive communication from surveillance. In today's digital world, this typically involves embedding secret data into innocuous digital content to enable covert exchanges. The characteristics of the medium significantly affect steganographic effectiveness, and the rise of AI-driven generative models, particularly diffusion models, presents promising opportunities. These models offer scalable mappings from textual prompts to diverse visual outputs, including images and videos, providing unique channels for steganographic embedding.

Recent steganographic methods leverage diffusion models by subtly perturbing pixel spaces to encode data~\citep{kim2024diffusionstego}. However, such pixel-space methods struggle to adapt to modern diffusion architectures, which increasingly operate in latent spaces. State-of-the-art latent diffusion models enhance generative quality but introduce additional decoding steps, complicating controlled embedding of hidden data. 

To address this challenge, we propose PSyDUCK, a novel, model-agnostic steganographic framework tailored explicitly for latent diffusion models. PSyDUCK embeds secret data through controlled divergence and local mixing during latent denoising, effectively overcoming the constraints imposed by latent-to-pixel transformations. Extensive evaluations show PSyDUCK's superior performance over existing methods, delivering robust, high-capacity steganographic solutions suitable for latent diffusion model deployments.

PSyDUCK assumes that sender and receiver share access to a latent diffusion-based generative system—a realistic scenario due to the widespread availability of open-source models—and possess common keys. These keys enable controlled modifications within the diffusion model's denoising trajectories, embedding messages securely and ensuring accurate reconstruction. Importantly, PSyDUCK can adapt embedding strength, optimizing the trade-off between message fidelity and detection resistance. This approach is robust against detection and tailored for latent diffusion models, addressing the fundamental challenges posed by latent-to-pixel space mappings. An overview of our method is illustrated in Figure~\ref{fig:teaser}.

Crucially, PSyDUCK is the first technique to extend steganographic encoding to latent video diffusion models. While previous diffusion-based steganography approaches have been limited to pixel-space image generation, PSyDUCK significantly advances this domain by supporting high-capacity, secure embedding within video content.

The primary contributions of this paper are as follows: \begin{itemize}
    \item We design a training-free steganographic method that is compatible with a wide range of latent diffusion models, enabling practical applications across diverse datasets and modalities.
    \item By analyzing the error properties of PSyDUCK, we demonstrate its ability to mitigate common vulnerabilities in steganography, such as detectability and decoding errors.
    \item We perform extensive experiments with both image and video data, comparing PSyDUCK against existing baselines to showcase its superior scalability and security.
\end{itemize}
Our results indicate that PSyDUCK is not only competitive with state-of-the-art methods in terms of accuracy and robustness but also establishes a new paradigm for steganography by achieving high channel capacity without retraining. This work provides a foundation for future research into general-purpose steganographic systems that leverage the growing power of generative diffusion models.

\section{Related Work}
\label{related_work}
%%%% Short section about diffusion models (ddpm, ddim, edm) %%%%
%%%%%%%%%%%%%%%%%%%%%  This should be in a background section, no? %%%%%%%%%%%%
% Diffusion models are a standard method for generating synthetic images and videos, operating by learning to iteratively denoise a sample drawn from a normal distribution.
% DDPM \citep{ho2020ddpm} is a commonly used network architecture that gradually refines a sample of the same dimension as the final output (i.e. the Pixel space).

%%%% Section about traditional message-hiding image steganography (hidden, etc.) %%%%
Traditional image steganography is predicated on the embedding of small, imperceptible perturbations into a chosen cover image.
Spatial methods such as least significant bit encoding \citep{schyndel1994lsb, wolfgang1996lsb}, adaptive methods such as HUGO, WOW, and S-UNIWARD \citep{pevny2010hugo, holub2012wow, holub2014suniward}, and transform-based methods like J-Steg and DCT \citep{provos2003jsteg, hetzl2005dct} offer varying levels of message capacity and are all known to be susceptible to detection via steganalysis \citep{denemark2016steganalysis, holub2015jpegsteganalysis, boroumandSteganalysis2019}.
Deep methods \citep{baluja2017deepstego, bui2023rosteals, cho2021stego, jing2021hinet, lu2021isn, xu2022riis, zhang2019steganoganhighcapacityimage, zhu2018hiddenhidingdatadeep, kumar2024videostegonographychaoticdynamicsystems} improve upon these techniques by learning to encode information into a cover image without greatly impacting visual quality while maintaining high decoding accuracy. 
These protocols can achieve much higher message capacities than their predecessors, but they require retraining and their use is also reliably detectable via steganalysis.
As these deep methods require retraining, they also implicitly require that the sender and receiver have private access to the trained models.

%%%% Section about message-hiding diffusion-based steganography (pulsar, stegaddpm, diffusionstego) %%%%
Alternatively, generative steganography has made way for coverless schemes, which eschew the need for shared cover images to secretly embed information.
\citet{kaptchuk2023pulsar} introduces Pulsar, which encodes and decodes messages using shared private keys as pseudorandom number generator (pRNG) seeds to guide stochastic Pixel diffusion models toward a desired sample. 
Pulsar suffers from a high error rate ($23\%$) in base form and thusly mitigates this with the use of customized error correction protocols. 
Its steganographic security is founded on the indistinguishability of pRNG outputs, but the method was found to be limited to pixel-based models and not be able to scale to latent-based models, partially due to their error correction protocols.
StegaDDPM \citep{peng2023stegaddpm} modifies the penultimate DDPM denoising step to use the bits of a message to create additional stochastic noise via a CDF transform. 
StegaDDPM enables higher bits-per-pixel (BPP) message capacity than Pulsar, but is highly sensitive to image degradation. 
Moreover, while the modified stochasticity mimics Gaussian distributions, biased message distributions compromise practical security.
DiffusionStego \citep{kim2024diffusionstego} introduced various message projection techniques to similarly embed the bits of a message into the added noise of the penultimate denoising step of pixel-based diffusion models. 
This achieves flexibility with multiple projection functions but struggles with high BPP, leading to sensitivity to noise and extraction errors. While the noise distributions match Gaussian statistics for uniform messages, the protocol is vulnerable if an adversary knows the projection functions.
Pulsar, StegaDDPM, and DiffusionStego all share similar key-based designs and offer some level of provable security. 
However, each of these methods are constrained to pixel-based models and show poor adaptability to latent diffusion models. Additionally, their channel capacities are far lower than those of deep methods, leaving significant room for improvement.

%%%% Section about image-hiding diffusion-based steganography (cross) %%%%
Among other diffusion-based generative steganography techniques is CRoSS \citep{yu2023cross}, which utilizes a deterministic, invertible diffusion process to hide entire images (as opposed to precise messages) into synthetic containers.
\citet{wei2023generativesteganographydiffusion} proposes retraining a diffusion model for the express purpose of steganography.
In the video space, \citet{mao2024videostego} proposes training an encoder/decoder network to hide messages into the latent space of video diffusion models.

%%%% Section about autoregressive generative steganography (chris's work, meteor, discop, ziegler et al) %%%%
Parallel to our study, proposals for autoregressive generative steganography have shown promise for secure communication. 
Stegasaurus \citep{ziegler2019neurallinguisticsteganography} uses arithmetic coding to map uniformly sampled text to the output of a large language model using the secret message to guide token selection.
Meteor \citep{kaptchuk2021meteor} continually XORs fresh pseudorandom masks with the secret message, using the result to sample tokens from an autoregressive generator.
These two approaches are provably secure up to the ability of an adversary to tell between real-or-random sequences \citep{kaptchuk2021meteor}.
Discop \citep{ding2023Discop} proposes the use of ``distribution copies," which rotates the probability intervals for token sampling while maintaining the same overall distribution.
\citet{dewitt2023perfectlysecuresteganography} identifies a correspondence between security, maximum encoding efficiency, and minimum entropy couplings for autoregressive models.
Both of these methods are \textit{perfectly secure}, in that the distributions of the unperturbed covertext and the stegotext are the same. While these proposals are all secure, they are disadvantaged with extremely low channel capacity (i.e. $< 7$ bits per token).

%\section{Background}
%\label{background}

\section{Framework}
\label{framework}
Here, we introduce the PSyDUCK framework. To extend steganography to \textit{latent}-based generative diffusion models, we propose a novel framework that enables precise control over the strength of an encoded message signal while maintaining imperceptibility. By dynamically adjusting the embedding process, PSyDUCK ensures that the encoded message remains robust against the distortions introduced by latent encoding and decoding schemes inherent to diffusion models, thereby improving reliability without increasing detectability. 

\subsection{Preliminaries}

Generative diffusion models, inspired by non-equilibrium thermodynamics, have emerged as a powerful framework for generating complex data distributions from simple noise~\citep{sohldickstein2015deepunsupervisedlearningusing}. These models operate through a probabilistic process consisting of a forward diffusion and a reverse denoising process.

The forward diffusion process incrementally adds Gaussian noise to an initial data point $\mathbf{x}_0$ over $T$ steps, producing a sequence $\mathbf{x}_1, \mathbf{x}_2, \ldots, \mathbf{x}_T$. Each step is modeled as:
\begin{equation}
    q(\mathbf{x}_t | \mathbf{x}_{t-1}) = \mathcal{N}(\mathbf{x}_t; \sqrt{1 - \beta_t} \mathbf{x}_{t-1}, \beta_t \mathbf{I}),
\end{equation}
where $\beta_t$ is the variance schedule that controls the noise level at step $t$. By the final step, $\mathbf{x}_T$ approximates a standard Gaussian distribution~\citep{sohldickstein2015deepunsupervisedlearningusing}.

% The reverse denoising process learns to transform noisy data back into the original data distribution by iteratively removing the added noise. A neural network is trained to predict the noise component at each step, enabling the recovery of $\mathbf{x}_{t-1}$ from $\mathbf{x}_t$.

The reverse denoising process learns to transform noisy data back into the original data distribution by iteratively removing the added noise. A neural network is trained to predict the noise component at each step, enabling the recovery of $\mathbf{x}_{t-1}$ from $\mathbf{x}_t$. The reverse process at each step can be modeled as:
\begin{equation}
    p(\mathbf{x}_{t-1} | \mathbf{x}_t) = \mathcal{N}(\mathbf{x}_{t-1}; \mu_\theta(\mathbf{x}_t, t), \sigma_\theta^2(t) \mathbf{I}),
\end{equation}
where $\mu_\theta(\mathbf{x}_t, t)$ represents the predicted mean and $\sigma_\theta^2(t)$ is the predicted variance, both parameterized by the neural network. During this process, the noise $\epsilon_t$ is sampled from the Gaussian distribution defined by $\mu_\theta(\mathbf{x}_t, t)$ and $\sigma_\theta^2(t)$ as:
\begin{equation}
    \epsilon_t \sim \mathcal{N}(0, \sigma_\theta^2(t)).
\end{equation}
The sampled noise is then used to adjust the latent variable in the denoising step.

Conditional diffusion models extend the generative framework by incorporating additional information (e.g., class labels or textual descriptions) to guide the generation process~\citep{dhariwal2021diffusionmodelsbeatgans}. 
% The reverse denoising step is modified to condition on $y$\FP{Do we use $y$ later?}, enabling the model to generate samples aligned with the specified condition. 
Latent diffusion models differ from pixel-based models by operating in a compressed latent space rather than directly on high-dimensional pixel data. Latent conditional diffusion models dramatically improve computational efficiency by operating in a compressed latent space rather than the high-dimensional data space~\citep{rombach2022highresolutionimagesynthesislatent}. The data $\mathbf{x}$ is first mapped to a latent representation $\mathbf{z}$ using an encoder. The diffusion process is then applied within the latent space. After completing the denoising process, a decoder reconstructs the data from the denoised latent representation.

\subsection{PSyDUCK Algorithm}
% \FP{The contribution here is not clear. We need to highlight exactly what is happening. Also we need a figure.}
% \AM{There's an error with the indices in the figure--should be from 0 to r-1.}

\begin{figure}[tp]
    \centering
    \includegraphics[width=0.48\textwidth]{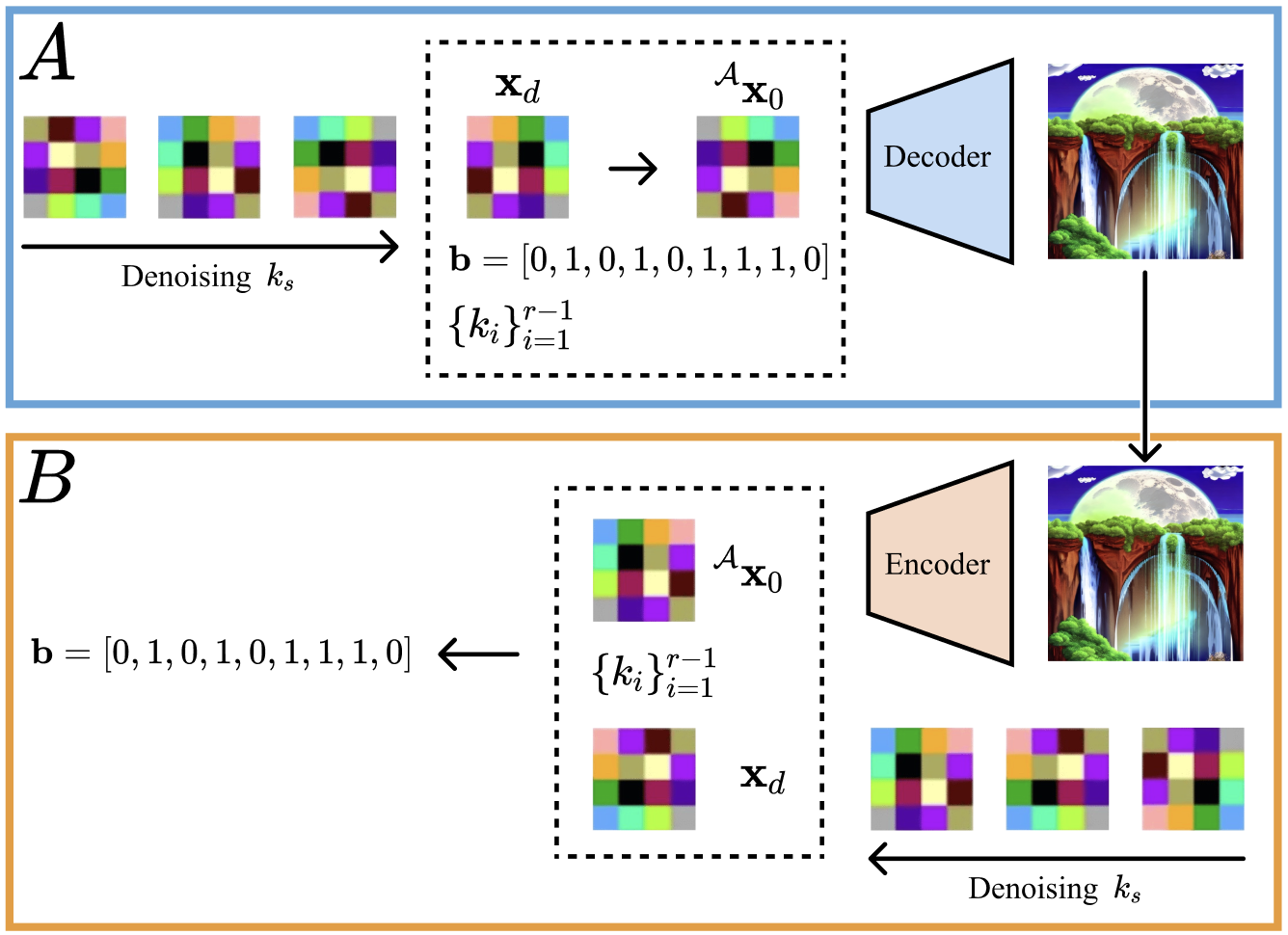}
    \caption{\textbf{Encoding and Decoding.} \small An illustration of the PSyDUCK encoding and decoding processes on latent model architectures. Dashed boxes denote custom PSyDUCK operations. To encode secret bitstring $\mathbf{b}$, Alice first denoises in the latent space until timestep $d$ with synchronization key $k_s$. Then, she diverges for $d$ steps using reference keys $k_i$ and subsequently mixes the diverged samples using $\mathbf{b}$. Finally, she puts her sample through the decoder to transmit a final output. To extract $\mathbf{b}$, Bob first encodes the transmission from Alice back into the latent space. He similarly denoises in the latent space until timestep $d$ with synchronization key $k_s$. Then, he diverges for $d$ steps using the reference keys $\{k_i\}_{i=0}^{r-1}$. Bob finally decodes Alice's message $\mathbf{b}$ by comparing his reference samples to Alice's transmission.}
    \label{fig:encoding_framework}
\end{figure}

In this section, we outline a new method, PSyDUCK, by which we extend training-free steganography to latent diffusion models. Alice and Bob wish to communicate steganographically using a diffusion model. They agree in advance on a synchronization key $k_s$ and $r$ reference keys \(\{k_i\}_{i=0}^{r-1}\) that will be used to create \emph{diverged samples} during the last \(d\) timesteps of the diffusion model's denoising process. They will use these keys to slightly alter the denoising trajectory at key points, introducing controlled divergence that facilitates the encoding and decoding of Alice's message.

\begin{definition}[Diverged Sample]
\label{def:diverged-sample}
Let \(\mathsf{DiffusionStep}(\cdot)\) denote one iterative step of the
diffusion model's denoising process. For a sample \(\mathbf{x}_t\) at
timestep \(t\) and a reference key \(k_i\), the \emph{diverged sample}
at timestep \(t-1\) is defined as:
\begin{equation}
    \label{eq:diverged-sample}
    \mathbf{x}^i_{t-1} \;:=\; \mathsf{DiffusionStep}(\mathbf{x}_t,\, t,\, k_i).
\end{equation}
If \(t = 1\), then \(\mathbf{x}^i_0\) is fully denoised and is thus called a \emph{reference sample}.
\end{definition}

\begin{definition}[Local \(\sf Mix\) Operation]
\label{def:mix}
Let \[
  X_t \;:=\; \bigl[\mathbf{x}_t^0,\;\dots,\;\mathbf{x}_t^{r-1}\bigr],
  \;\text{where } \mathbf{x}_t^i \;=\; \bigl\{x_{t,j}^i \mid j \in \mathcal{I}\bigr\},
\]
be a set of diverged samples at timestep \(t\), and let \(\mathbf{b} = [\,b_1,\;b_2,\;\dots,\;b_l\,]\)
be a bitstring of length \(l\). 
We define the local mixing operator \(\mathsf{Mix}(\mathbf{b}, X_t)\) as follows:
\begin{equation}
\label{eq:mix}
  \mathsf{Mix}(\mathbf{b}, X_t)
  \;:=\;
  \bigl[x_{t,1}^{\,b_1},\;x_{t,2}^{\,b_2},\;\dots,\;x_{t,l}^{\,b_l}\bigr].
\end{equation}
In other words, for each position \(j \in \{1,\dots, l\}\), the bit \(b_j\) selects which diverged sample \(\mathbf{x}_t^{\,b_j}\) contributes to each local patch or pixel \(x_{t,j}^{\,b_j}\).
\end{definition}

Let the timesteps of the diffusion model be indexed by \(t\) decreasing from \(T\) to \(0\), where \(t=0\) represents a fully denoised sample. The integer \(d\) indicates which of these later timesteps (\(t = d,\,\ldots,\,1\)) are used for divergence. In PSyDUCK, both parties obtain an intermediate sample \(\mathbf{x}_{d+1}\) via typical denoising using synchronization key $k_s$. From this point on, they use each reference key \(k_i\) to diverge from the original trajectory.

To embed a secret bitstring \(\mathbf{b} = [\,b_1,\dots,b_l]\), Alice first obtains the diverged samples \(\{\mathbf{x}_1^i\}_{i=0}^{r-1}\) (by denoising according to each \(\{k_i\}_{i=0}^{r-1}\) for $d$ iterations) and then \emph{mixes} them into a single sample using the local mixing operator \(\mathsf{Mix}\). She then denoises a final time to obtain a sample $\alice{\bx_0}$.

Prior to transmitting $\alice{\bx_0}$, Alice may perform some postprocessing on the sample.
This is not required if the diffusion model of choice operates in directly in the pixel space, so Alice can directly transmit $\alice{\bx_0}$ in such case.
However, for latent diffusion models, diverging and mixing both occur in latent space, so $\alice{\bx_0}$ requires decoding.
Let \(\mathrm{Enc}(\cdot)\) and \(\mathrm{Dec}(\cdot)\) refer to the respective encoder and decoder networks of the latent diffusion model.
Alice can finally transmit $\mathrm{Dec}\bigl(\alice{\bx_0}\bigr)$ to Bob. 
% \begin{enumerate}
%     \item \emph{Alice’s Post-Processing:}
%     After generating her steganographic latent sample \(\alice{\bx_0}\) (via diffusion and mixing in the latent space), Alice decodes it into the visible domain:
%     \[
%        \alice{\bx_{\mathrm{vis}}} \;=\; \mathrm{Dec}\bigl(\alice{\bx_0}\bigr).
%     \]
%     This decoded image (or other pixel-space representation) is what Alice transmits.
%     \item \emph{Bob’s Re-Encoding:} 
%     Upon receiving \(\alice{\bx_{\mathrm{vis}}}\), Bob encodes it back into the latent space:
%     \begin{equation}
%          \alice{\bx_0} \;=\; \mathrm{Enc}\bigl(\alice{\bx_{\mathrm{vis}}}\bigr),
%     \end{equation}~such that he can compare Alice's final latent sample with his own to recover \(\mathbf{b}\).
% \end{enumerate}

On Bob's end, he obtains \(\alice{\bx_0}\) (either directly from Alice's transmission or via $\mathrm{Enc}(\cdot)$ to undo Alice's postprocessing).
To extract $\mathbf{b}$, Bob repeats the same process of divergence using his copies of the reference keys $\{k_i\}_{i=0}^{r-1}$, but for a total of $d+1$ steps.
He is left with a set of reference samples that aggregate to \(\bob{X_0}\). Bob knows \(\mathbf{b}\) must be some element of \(\mathbb{B}^\ell\) (length-\(\ell\) bitstrings), so he recovers \(\mathbf{b}\) by finding the bitstring \(\bbh\) that best “matches” Alice’s final sample:
\[
  \bbh \;=\;
  \argmin_{\bb \in \mathbb{B}^\ell}
    \bigl\|\mathsf{Mix}(\bb,\,\bob{X_0}) \;-\; \alice{\bx_0}\bigr\|_{2}.
\]
By the local nature of \(\sf Mix\), Bob can thus be guaranteed to identify $\bbh$ in linear time.
For more fine-grained reference, pseudocode for the PSyDUCK encode and decode processes may be found in Appendix~\ref{appendix:algorithms}.

Figure~\ref{fig:encoding_framework} presents an overview of the PSyDUCK encoding and decoding processes. An illustration of the divergent trajectories and the local $\sf Mix$ operation where the number of divergent steps $d=2$ and the number of reference keys $r=2$ may be found in Figure~\ref{fig:two-subfigs}.

\begin{figure}[t]
    \centering
    
    %--- Subfigure 1 ---
    \begin{subfigure}[b]{0.44\linewidth}
        \centering
        \includegraphics[width=0.99\linewidth]{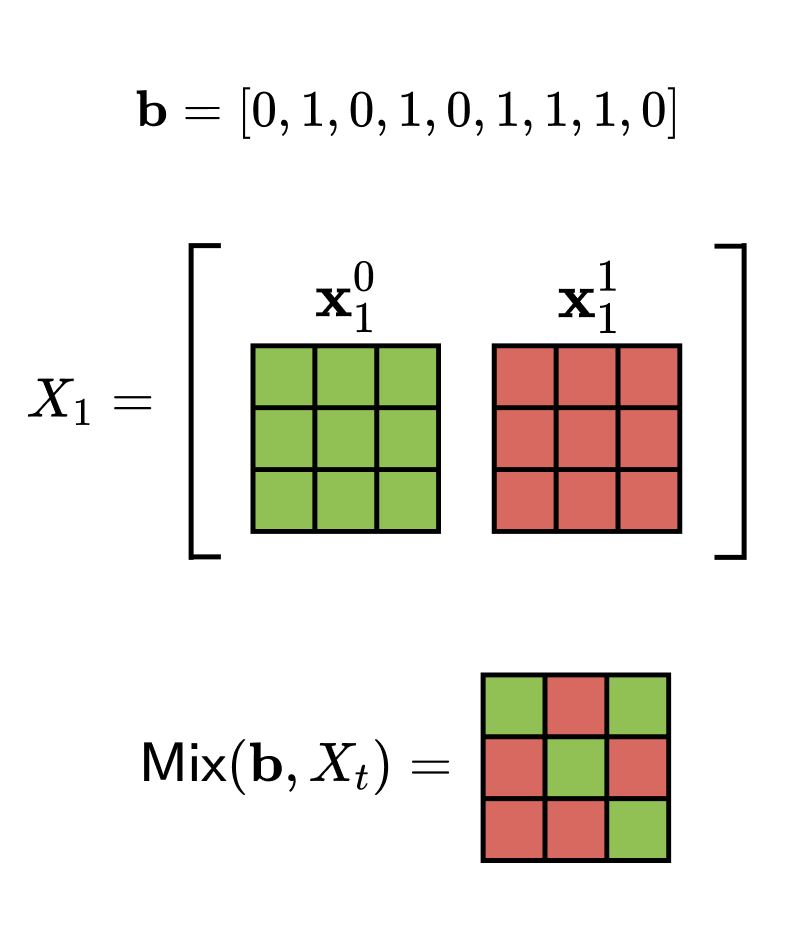}
        \caption{The $\sf Mix$ operation.}
        \label{fig:sub:left}
    \end{subfigure}
    \hfill
    \begin{subfigure}[b]{0.55\linewidth}
        \centering
        \includegraphics[width=0.99\linewidth]{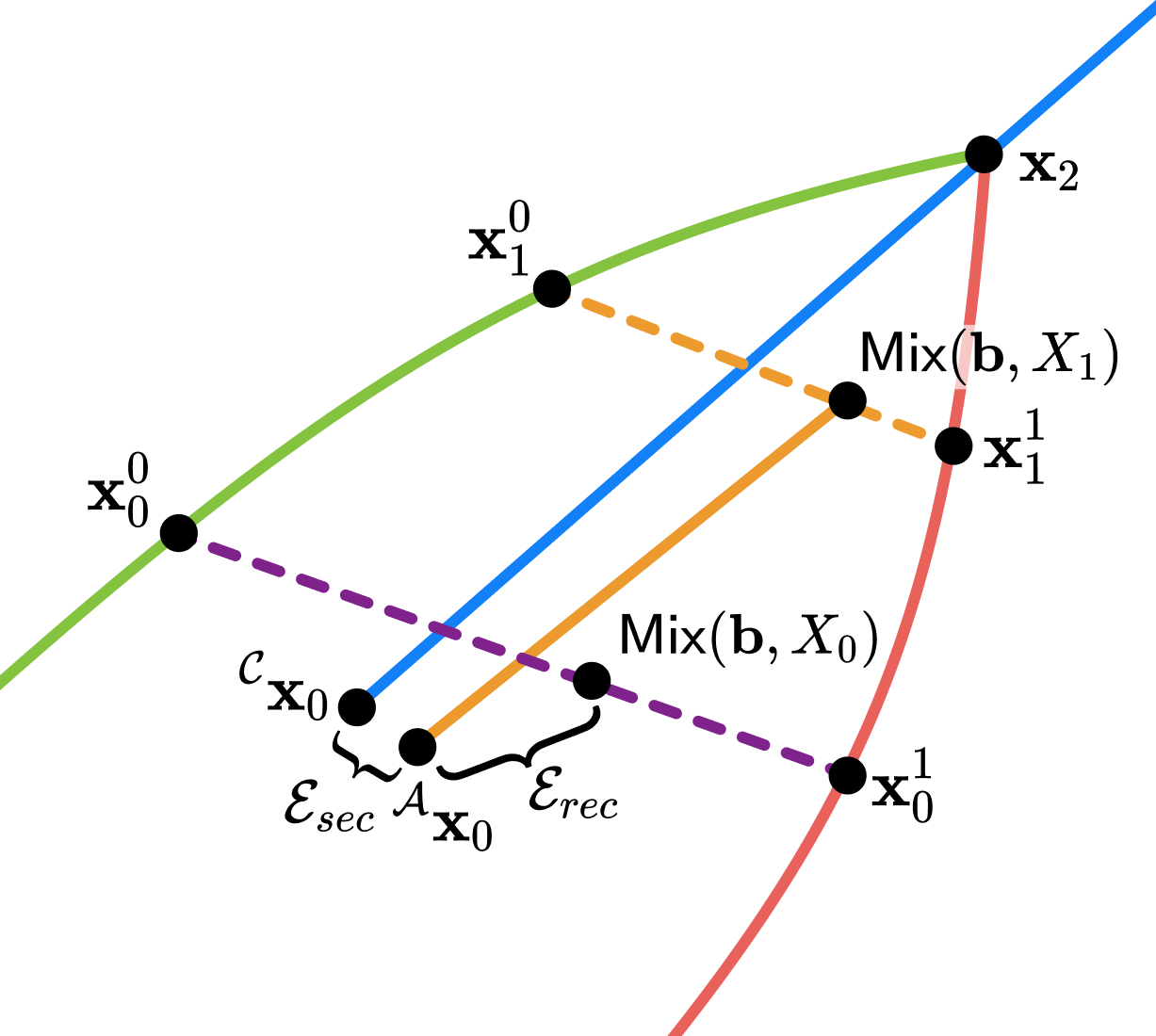}
        \caption{Diverging and mixing.}
        \label{fig:sub:right}
    \end{subfigure}
    % \hfill
    %--- Subfigure 2 ---

    %--- Overall caption ---
    \caption{\textbf{Diverging and Mixing.}
        \subref{fig:sub:left} An example of mixing the red and green trajectories based on $\mathbf{b}$ to form a mixed sample when $t=1$ and $r=2$.  
        \subref{fig:sub:right}
        Here, the blue path represents the original trajectory 
        of the diffusion model, resulting in 
        $\cover{\bx_0}$. The green path represents the divergent trajectory when conditioned with $k_0$, while the red path presents the divergent 
        trajectory when conditioned with $k_1$. The orange path shows Alice's trajectory upon mixing samples 
        from the two divergent paths.}
    \label{fig:two-subfigs}
\end{figure}

\subsection{Optimization} 

We consider two metrics for evaluating PSyDUCK: the security error \(\mathcal{E}_{sec}\) and the reconstruction error \(\mathcal{E}_{rec}\). The security error measures deviations from the original diffusion trajectory and is defined as
\[
  \mathcal{E}_{sec} = \| \alice{\bx_0} - \bx_0 \|.
\]
Minimizing \(\mathcal{E}_{sec}\) reduces detectability but limits the encoding capacity, since shorter divergence steps or fewer reference samples must be used. This definition of security error is useful for exploring the security guarantees of PSyDUCK, which we discuss in Section~\ref{theory}. However, in Section~\ref{experiments}, we evaluate the practical security of PSyDUCK with the detection rates of two well-known steganalyzers. 

The reconstruction error quantifies the accuracy with which the receiver recovers the embedded bitstring:
\[
  \mathcal{E}_{rec} = \| \alice{\bx_0} - \sf{Mix}(\mathbf{b}, \bob{X_0}) \|,
\]
where \(\sf {Mix}(\mathbf{b}, \bob{X_0})\) denotes the local mixing of reference samples according to Alice's message. Minimizing \(\mathcal{E}_{rec}\) improves decoding reliability but typically requires larger divergence steps or more reference samples, which in turn may increase \(\mathcal{E}_{sec}\).

% \section{Properties}
% \label{properties}
% \input{sections/4_properties}

\section{Theoretical Analysis}
\label{theory}
 We show two key results about the PSyDUCK framework:
\begin{enumerate}
    \item Bounded Security Error (Proposition~\ref{prop:err-sec-bound-formal}). If the noise estimator $\epsilon_\theta$ is bounded, then the difference $\|\alice{\bx_0} - \cover{\bx_0}\|$ between the steganographic sample $\alice{\bx_0}$ and the corresponding cover sample $\cover{\bx_0}$ is also bounded. Concretely, we show: 
    \[
        \|\alice{\bx_0} - \cover{\bx_0}\| \,\in\, O\bigl(d \,r\, \sigma_{d+1}\bigr).
    \]
    \item Security via Indistinguishable Noise (Proposition~\ref{prop:sto-d1-secure-formal}). When a diffusion model injects noise \emph{at the end} of each denoising step (i.e., a first-order SDE solver such as DDPM or DDIM), the security of PSyDUCK reduces to an adversary's ability to distinguish ``random noise'' from ``a mixture of random noise.'' If the mixture also appears random, the method is secure against detection.
\end{enumerate}

\begin{proposition}
\label{prop:err-sec-bound-formal}
(Bounded Security Error)\\
Let $\epsilon_\theta$ be a noise estimator that is bounded. That is, there exists a constant $C>0$ such that, for all relevant $\bx$ and $t$,
\[
  \|\epsilon_\theta(\bx, t)\| \;\le\; C.
\]
Then, under the PSyDUCK framework,
\[
  \|\alice{\bx_0} - \cover{\bx_0}\|
  \;\in\; O\bigl(d \, r \,\sigma_{d+1}\bigr).
\]
\end{proposition}

We show with Lemmas~\ref{lem:noise-bounded}, \ref{lem:dstep-bounded}, and \ref{lem:mix-bounded} that each diverged sample is within $O(d\,\sigma_{d+1})$ of $\bx_{d+1}$. Thus, elementwise mixing among $r$ such samples incurs an $O(r\; d \; \sigma_{d+1})$ deviation. 
\begin{proof}[Abridged Proof of Proposition~\ref{prop:err-sec-bound-formal}]
By Lemma~\ref{lem:noise-bounded}, the step from $\alice{\bx_0}$ to $\mathsf{Mix}(\bb,X_1)$ is on the order of $\sigma_1$. Next, applying Lemma~\ref{lem:mix-bounded} shows that mixing diverged samples and comparing to $\alice{\bx_{d+1}}$ introduces an additional $O(r\,d\,\sigma_{d+1})$ difference. Finally, Lemma~\ref{lem:dstep-bounded} shows that going from $\alice{\bx_{d+1}}$ to $\cover{\bx_0}$ adds another $O(d\,\sigma_{d+1})$. Combining these via the triangle inequality,
\begin{equation}
    \begin{split}
        \|\alice{\bx_0} - \cover{\bx_0}\| &\;\in\;  O(\sigma_1) + O(r\,d\,\sigma_{d+1}) + O(d\,\sigma_{d+1}) \\
         &\;\subseteq\; O\bigl(d\,r\,\sigma_{d+1}\bigr).
    \end{split}
\end{equation}
Since $\sigma_1$ can be treated as a small or absorbed constant for large $d$, the final bound is $O(r\,d\,\sigma_{d+1})$.
\end{proof}

Secondly, we state and prove that for a first-order SDE-based diffusion model---where noise is appended after the denoising step---any mixture of noise remains indistinguishable from pure noise.

\begin{proposition}
\label{prop:sto-d1-secure-formal}
(Stochastic Security)\\
Let $\mathrm{DiffusionStep}(\bx_t, t, k_i) = \mathrm{Model}(\bx_t, t, k_s) + \sigma_t \,\epsilon_i$, where $\mathrm{Model}(\cdot)$ is a deterministic denoiser keyed by $k_s$ and $\epsilon_i$ is noise generated by key $k_i$ appended after the denoising step.
Security reduces to whether an adversary can distinguish honest noise $\epsilon$ from a mixed noise
$\mathsf{Mix}(\bb,\{\epsilon_0,\ldots,\epsilon_{r-1}\})$. If they are indistinguishable, the scheme is secure.
\end{proposition}

\begin{proof}[Abridged Proof of Proposition~\ref{prop:sto-d1-secure-formal}]
A standard diffusion step at $t=1$ gives:
\[
  \cover{\bx_0} 
  = \mathrm{DiffusionStep}(\bx_1, 1, k_s)
  = \mathrm{Model}(\bx_1) + \sigma_1\,\epsilon.
\]
Under Psyduck with $d=1$, we mix noise at $t=2$:
\begin{equation}
    \begin{split}
        \alice{\bx_1} &= \mathsf{Mix}\bigl(\bb,\{\bx_1^0,\dots,\bx_1^{r-1}\}\bigr), \\
        \text{where} & \quad 
  \bx_1^i = \mathrm{Model}(\bx_2) + \sigma_2\,\epsilon_i.
    \end{split}
\end{equation}
  
By linearity of $\mathsf{Mix}$, 
\[
  \alice{\bx_1}
  = \mathrm{Model}(\bx_2) + \sigma_2 \,\mathsf{Mix}\bigl(\bb,\{\epsilon_0,\dots,\epsilon_{r-1}\}\bigr).
\]
Then the final sample at $t=1$ is
\begin{equation}
    \begin{split}
         \alice{\bx_0} 
  &= \mathrm{DiffusionStep}(\alice{\bx_1}, 1, k_s) \\
  &= \mathrm{Model}\bigl(\alice{\bx_1}\bigr) + \sigma_1\,\epsilon.
    \end{split}
\end{equation}
 
In both cases, $\mathrm{Model}$ is deterministic, and the final term $\sigma_1\,\epsilon$ appears as fresh noise. Hence any difference between $\cover{\bx_0}$ and $\alice{\bx_0}$ comes solely from whether $\mathsf{Mix}(\bb,\{\epsilon_0,\dots,\epsilon_{r-1}\})$ is distinguishable from an i.i.d.\ noise sample $\epsilon$. If not, no adversary can detect the presence of steganography.
\end{proof} 

Extended proofs and relevant lemmas can be found in Appendix~\ref{appendix:proofs}.

\section{Experiments}
\label{experiments}
We evaluate the performance of the PSyDUCK framework across various steganographic tasks, including both image and video applications. Our experiments cover pixel-based and latent-based diffusion models, with a focus on how controlled divergence affects recovery accuracy, encoding capacity, and detection rates. We first compare PSyDUCK to existing baselines to assess its effectiveness and scalability. Next, we examine its performance in video-based steganography. Finally, ablation studies are conducted to analyze the influence of key factors such as divergent step counts, model precision, and base image type. 

\subsection{Experimental Details}
%\FP{@georgia even use bolding for a preliminary thing in captions or don't, stick to one decision. The bolded thing should be very short. Put all captions at the bottom regardless of table/figure. Bold the text of each table cell used as heading}
\paragraph{Models}

In our experiments with pixel-based image steganography, we use four open-source diffusion models from the DDPM~\citep{ho2020ddpm} paper: \texttt{celeb}, \texttt{bedroom}, \texttt{cat}, and \texttt{church}. For latent-based image steganography, we use Stable Diffusion (SD) version $2.1$, conditioned on text inputs~\cite{rombach2022highresolutionimagesynthesislatent}.
%\FP{Homogenize in the text, use only SD and for stable video diffusion use SVD} 
%
For video experiments, we use Stable Video Diffusion (SVD), an image-to-video latent model we condition with ImageNet samples.
%
% We provide relevant text and image prompts in Appendix~\ref{appendix:experimental-details}.
\begin{table}[t]
    \centering
    \resizebox{\columnwidth}{!}{%
    \begin{tabular}{c c c c c c}
        \toprule
        \multirow{2}{*}{\textbf{Model Type}} & \multirow{2}{*}{\textbf{Stegosystem}} & \multirow{2}{*}{\textbf{Bytes}} & \multirow{2}{*}{\textbf{Acc.}} & \multicolumn{2}{c}{\textbf{Detection Rate}} \\
        \cmidrule(lr){5-6}
        & & & & \textbf{SRNet} & \textbf{SiaStegNet} \\
        \midrule
        \multirow{8}{*}{Pixel-based} 
        & Pulsar & 541.7 & 94.00 & 50.0 & 50.0 \\
        & DiffusionStego & 8192 & \textbf{98.46} & 76.3 & 85.0 \\
        & StegaDDPM & 512 & 88.62 & 50.0 & 50.0 \\
        \cmidrule(lr){2-6}
        & Psyduck ($d = 1$) & 512 & 92.95 & \textbf{50.0} & \textbf{50.0} \\
        & Psyduck ($d = 2$) & 512 & 97.47 & 56.3 & 50.0 \\
        & Psyduck ($d = 3$) & 1536 & 96.47 & 57.4 & 55.0 \\
        & Psyduck ($d = 10$) & 8192 & 95.63 & 76.3 & 80.0 \\
        \midrule
        \multirow{6}{*}{Latent-based}
        % & Pulsar & - & - & - & - \\
        & DiffusionStego & 32 & 72.38 & 74.6 & 75.0 \\
        & StegaDDPM & 32 & 67.47 & 50.0 & 50.0 \\
        \cmidrule(lr){2-6}
        & Psyduck ($d = 1$) & 96 & 94.90 & \textbf{50.0} & \textbf{50.0} \\
        & Psyduck ($d = 2$) & 96 & 97.77 & 51.2 & 50.0 \\
        & Psyduck ($d = 3$) & 96 & \textbf{98.42} & 51.5 & 60.9 \\
        & Psyduck ($d = 10$) & 512 & 94.65 & 84.2 & 85.0 \\
        \bottomrule
    \end{tabular}}
    \caption{\textbf{Comparison of pixel-based and latent-based models.} \small A comparison of training-free steganography schemes across pixel- and latent-based approaches. Reported metrics include bytes encoded, transmission accuracy (Acc.), and detection rates (SRNet, SiaStegNet) for various methods, including our Psyduck stegosystem with varying divergence levels $d$.}
    \label{tab:latent-comp}
\end{table}

\paragraph{Metrics}

To evaluate the performance of PSyDUCK, we report the number of bytes encoded per frame, the accuracy of transmission, and the rate of detection by steganalysis tools for each of our experiments. Higher values of bytes encoded and accuracy of transmission indicate better performance, while lower detection rates indicate better performance. A $50.0\%$ detection rate indicates that the steganalysis tool is unable to perform better than random guessing. We use two well-known steganalyzers, SRNet \citep{wu2019editingtextwild} and SiaStegNet \citep{youSteganalysis2020}. 

\paragraph{Data and Prompts}

To generate stegosamples with SD v2.1 for latent image steganography, we use a series of prompts, which are reproduced in Appendix~\ref{appendix:experimental-details}. To generate video stegosamples with SVD, instead, we randomly sample 100 base images from ImageNet~\citep{imagenet} and use them as input of the text-to-video system. For each experiment, we independently train a steganalyzer, retraining it for each hyperparameter set using cover and stegoimages from the corresponding diffusion model to improve its ability to detect hidden information. Since the steganalyzers are not designed for video, we preprocess the video experiments by extracting individual frames and then train and evaluate the models frame by frame, treating them as independent images.

\subsection{Image Steganography}
% \FP{Important question: how $d$ and the bytes encoded are related, if they are? How do we extract a relationship between them? How do we decide how many bytes to encode?}
\begin{figure}[t]
    \centering
    \begin{subfigure}{\linewidth}
    \setlength{\tabcolsep}{2pt}
    \resizebox{\linewidth}{!}{%
    \begin{tabular}{cccccc}
        &\makecell{\large The armaments \\[-0.3em] \large have moved.} & 
    \makecell{\large Ain't I \\[-0.3em] \large a woman?} & 
    \makecell{\large The revolution will \\[-0.3em] \large not be televised.} & 
    \makecell{\large The Dude abides.} & 
    \makecell{\large A single spark can \\[-0.3em] \large start a prairie fire.} \\

        \raisebox{2.7em}{\rotatebox{90}{\Large{Cover}}}& \includegraphics[width=3cm]{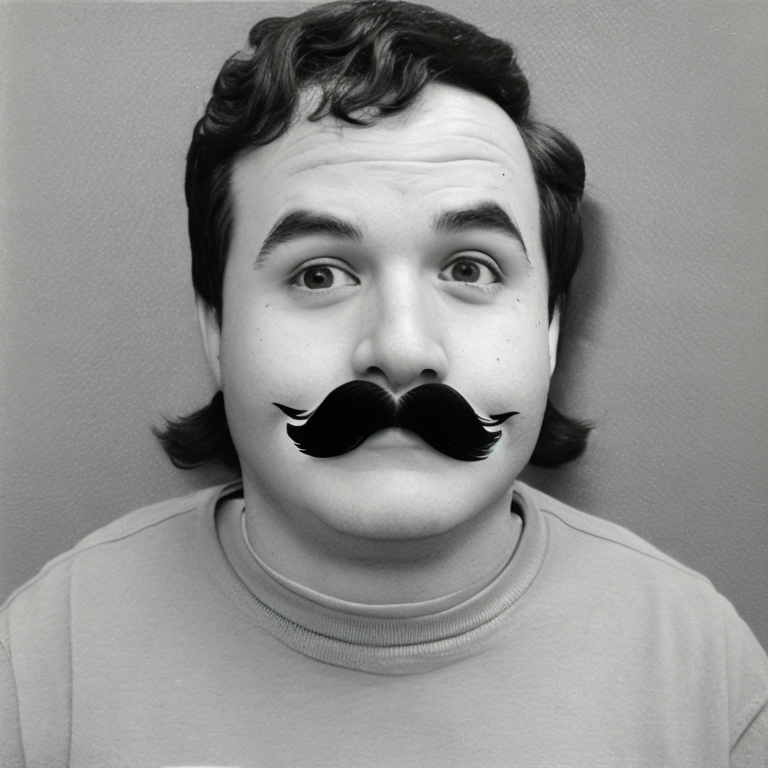} &
        \includegraphics[width=3cm]{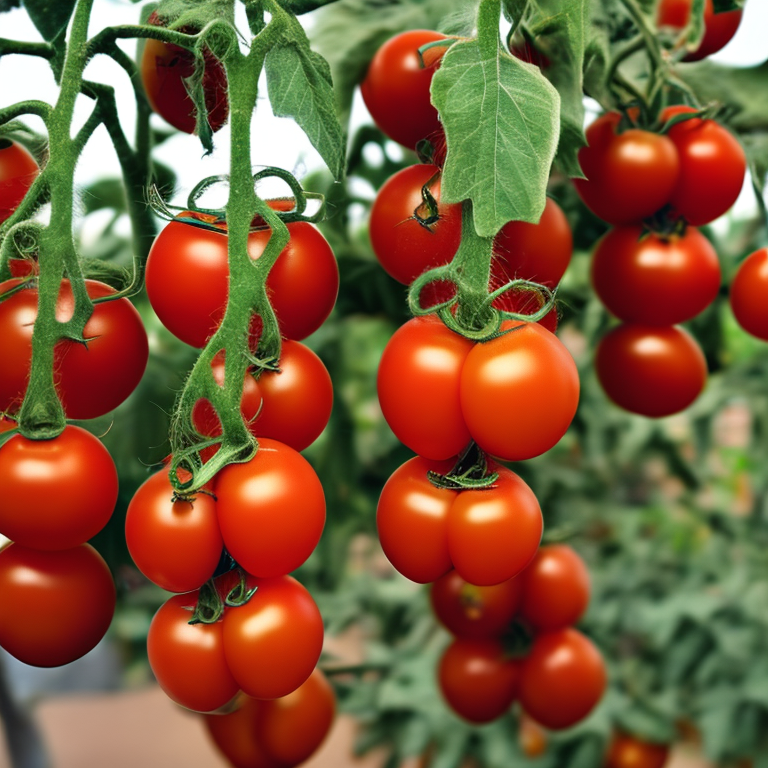} &
        \includegraphics[width=3cm]{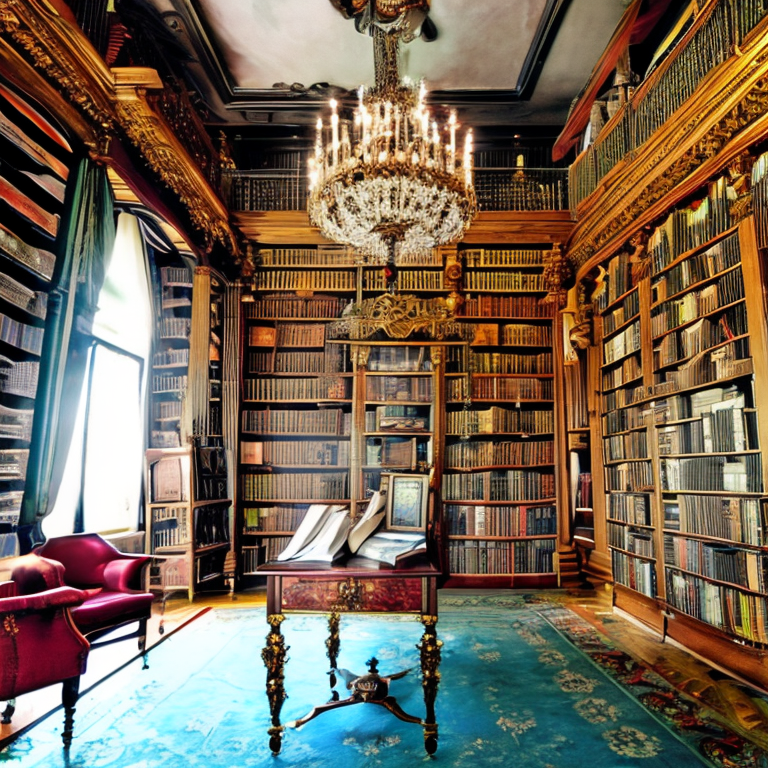} &
        \includegraphics[width=3cm]{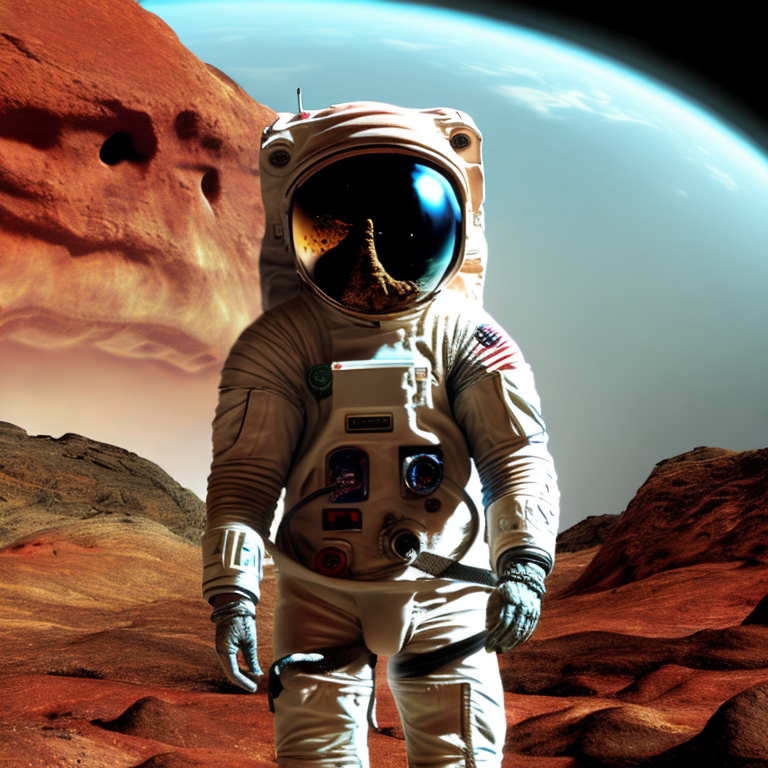} &
        \includegraphics[width=3cm]{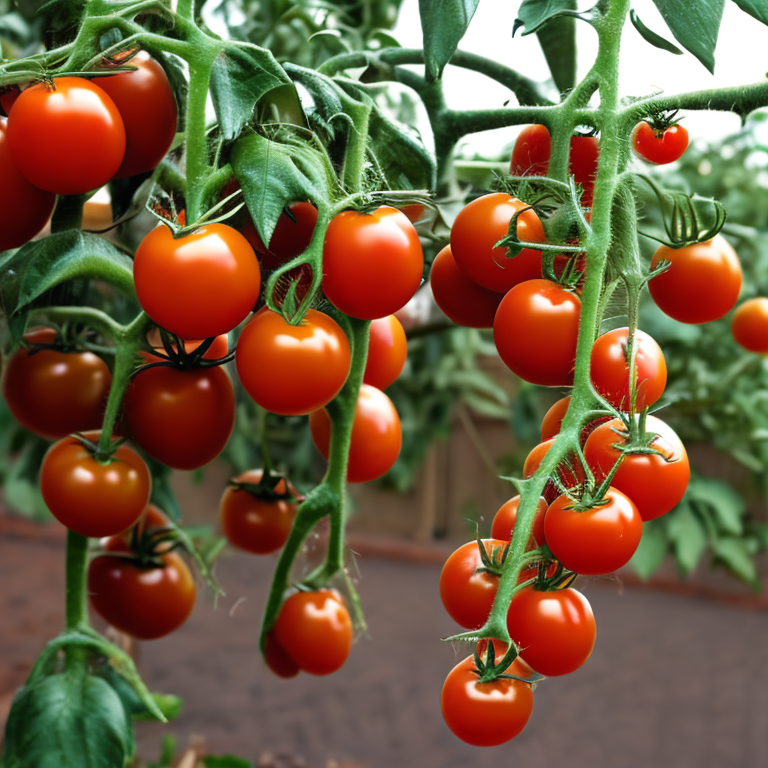} \\
        \raisebox{2.7em}{\rotatebox{90}{\Large{Stego}}}&
        \includegraphics[width=3cm]{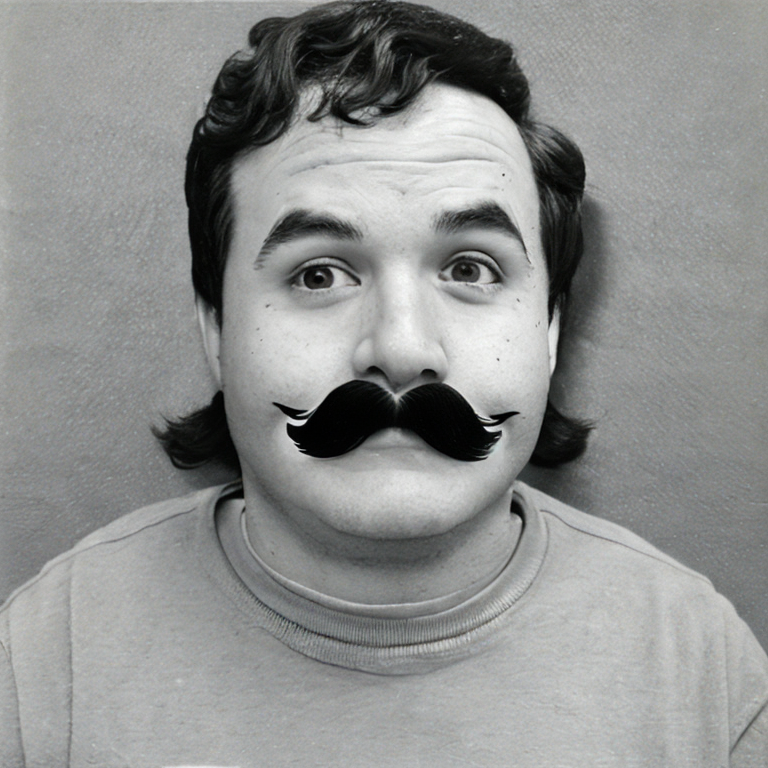} &
        \includegraphics[width=3cm]{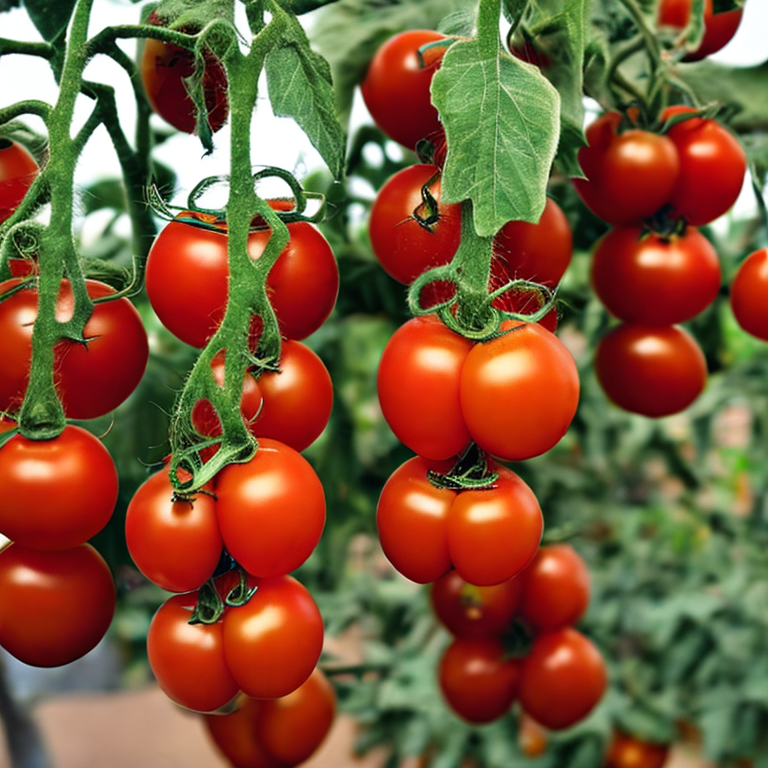} &
        \includegraphics[width=3cm]{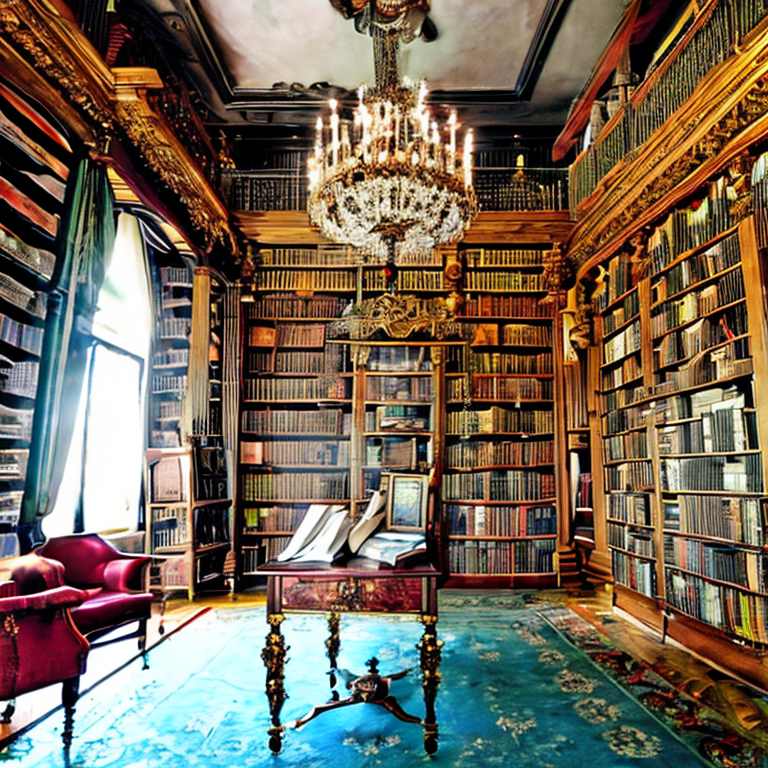} &
        \includegraphics[width=3cm]{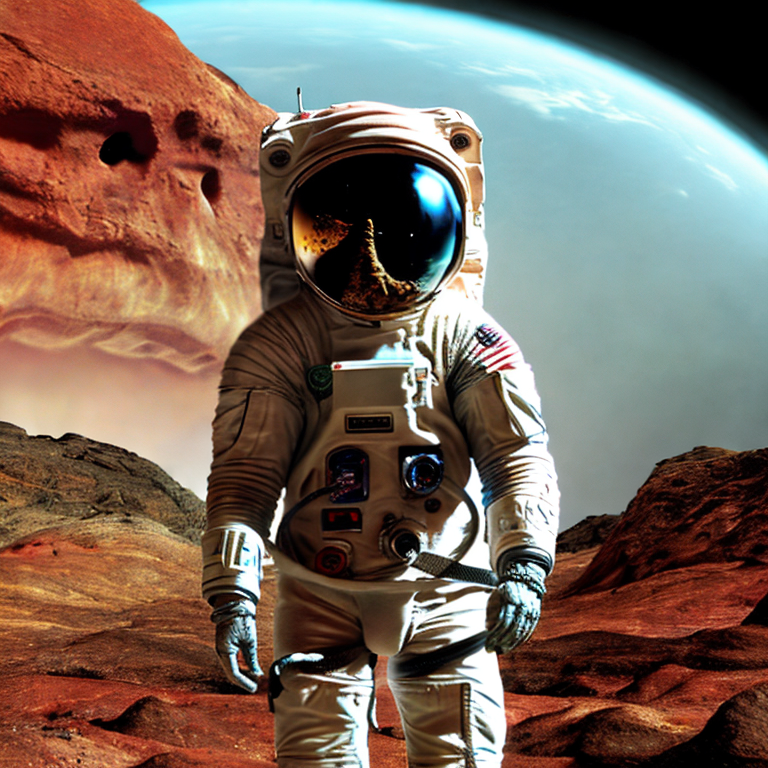} &
        \includegraphics[width=3cm]{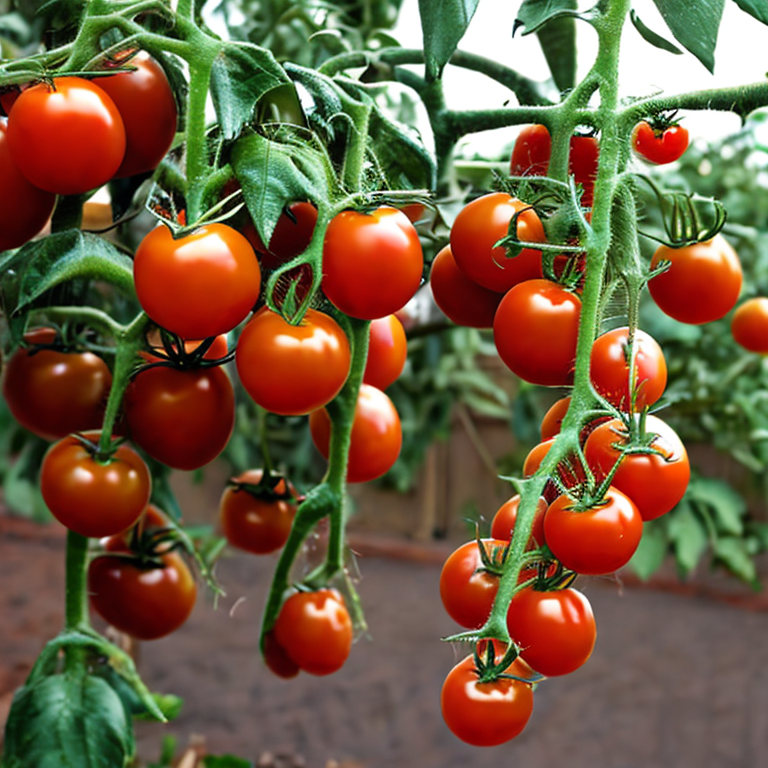} \\
    \end{tabular}}
    \caption{Cover and stegoimages}\label{fig:qual-img}

    \end{subfigure}

    \begin{subfigure}{\linewidth}
    \setlength{\tabcolsep}{2pt}
    \resizebox{\linewidth}{!}{%
    \begin{tabular}{cccccc}
    &\Large Cover & \Large$d=1$ & \Large$d=3$ & \Large$d=10$ & \Large$d=20$\\
    \raisebox{2.7em}{\rotatebox{90}{\Large Image}}&
        \includegraphics[width=3cm]{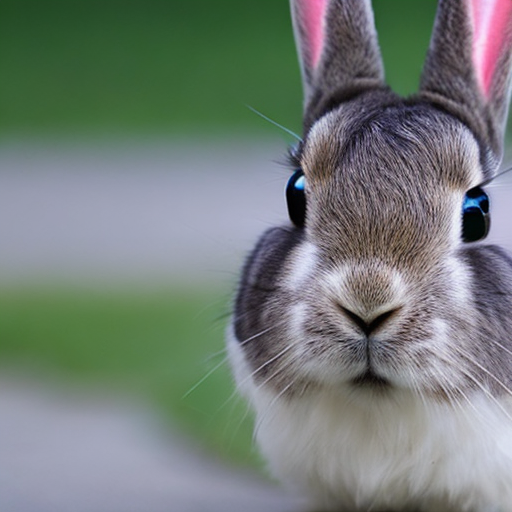} &
        \includegraphics[width=3cm]{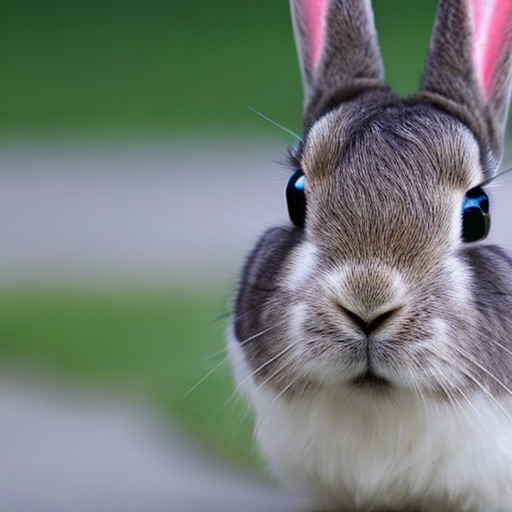} &
        \includegraphics[width=3cm]{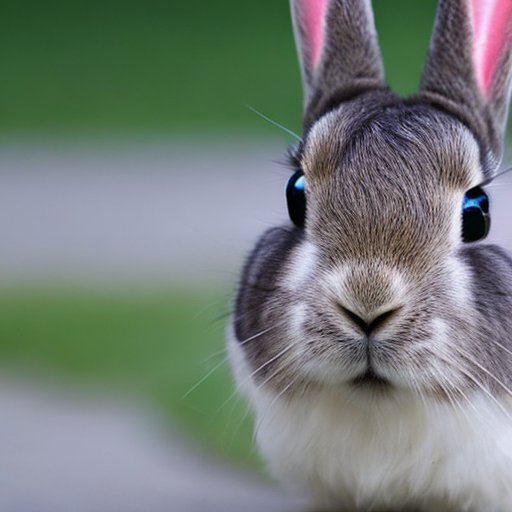} &
        \includegraphics[width=3cm]{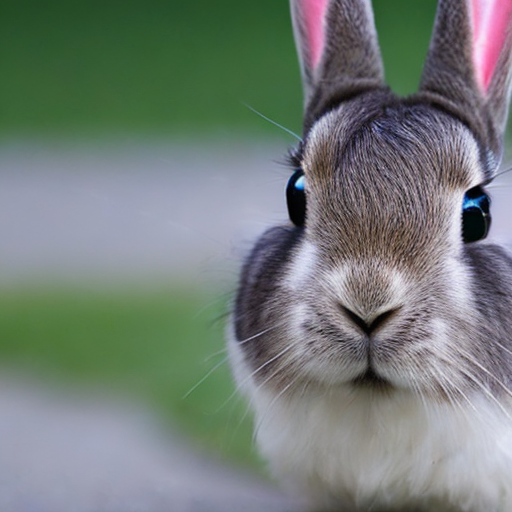} &
        \includegraphics[width=3cm]{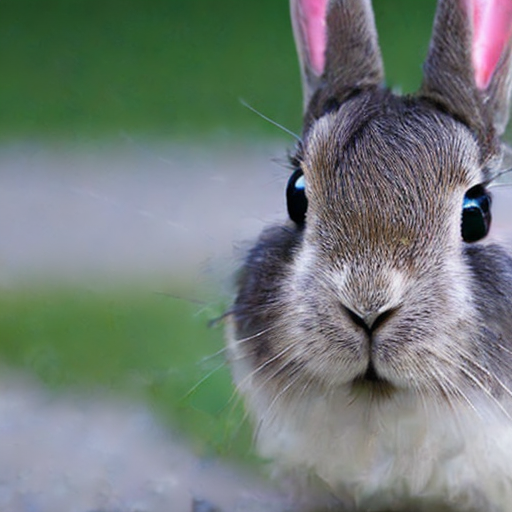} \\ 
        &&
        \includegraphics[width=3cm]{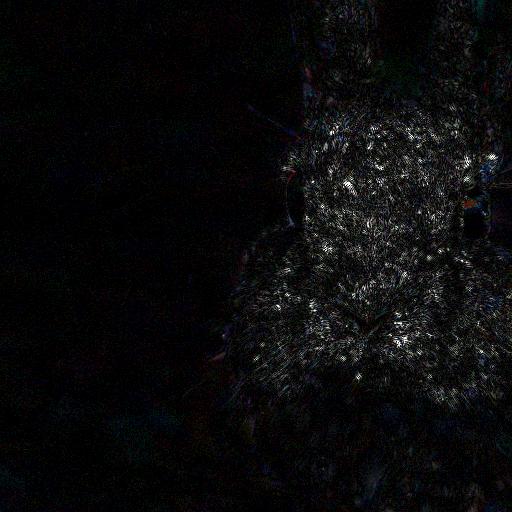} &
        \includegraphics[width=3cm]{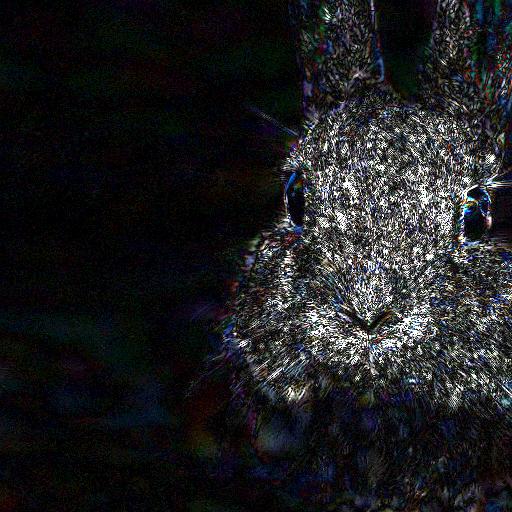} &
        \includegraphics[width=3cm]{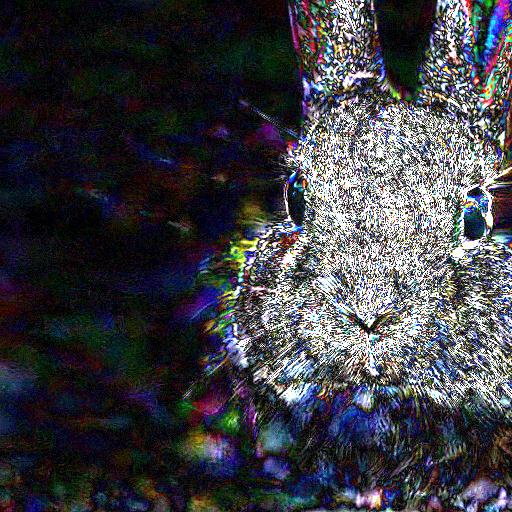} &
        \includegraphics[width=3cm]{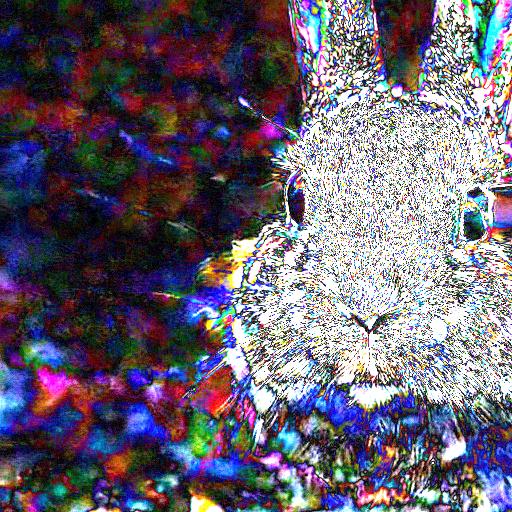} \\
    \end{tabular}}
    \caption{Effects of $d$}\label{fig:qual-d}
    \end{subfigure}
    \caption{\textbf{Qualitative analysis of SD v2.1 stegosamples.} \subref{fig:qual-img} Examples of stegoimages obtained by encoding the text reported at the top on the cover image with PSyDUCK and SD v2.1. Stegoimages are perceptually undistinguishable. \subref{fig:qual-d} Images generated from SD v2.1 using identical keys but varying divergent step count $d$. For easing visualization, we display the differences between cover and stego image magnified by $20\times$ in the bottom row.}
    \label{fig:qual-img-complete}
\end{figure}

\paragraph{Quantitative Evaluation}
To ease comparisons with existing baselines, we begin by presenting experiments on established image steganography methods. Specifically, we evaluate against Pulsar~\citep{kaptchuk2023pulsar}, StegaDDPM~\citep{peng2023stegaddpm}, and DiffusionStego~\citep{kim2024diffusionstego}. To demonstrate the versatility of PSyDUCK across different scenarios, we first compare its performance to the baselines using \textit{pixel}-based diffusion models. Since we are encoding bitstrings, we use only two reference keys ($r = 2$) throughout the subsequent experiments. The number of bytes encoded in each model is a function of the number of divergent steps $d$, the number of reference keys $r$, and the size of the model’s embedding space, whether in pixel-space for \textit{pixel}-based models or latent-space for \textit{latent}-based models. The results, summarized in Table~\ref{tab:latent-comp}, indicate that PSyDUCK closely matches the performance of Pulsar when the number of divergent steps $d=2$. PSyDUCK achieves slightly higher recovery accuracy but with a marginally higher detection rate by the SRNet steganalyzer. DiffusionStego, on the other hand, demonstrates significantly higher throughput and recovery accuracy, at the cost of substantially increased detection rates. 

The true potential of PSyDUCK is demonstrated through experiments with \textit{latent}-based diffusion models. We evaluate PSyDUCK performance using SD v2.1 as the model backbone and compare it against baseline methods. As shown in Table~\ref{tab:latent-comp}, PSyDUCK consistently outperforms the baselines in both throughput and recovery accuracy. Specifically, for $d = 1$, PSyDUCK achieves 94.90\% recovery accuracy while maintaining low detection rates of 50.0\% by the SRNet steganalyzer, matching StegaDDPM’s detection rates but with significantly higher throughput. Increasing the number of divergent steps to $d = 2$ further boosts recovery accuracy to 97.77\%, with negligible impact on detection rates (51.2\% and 50.0\%). The maximum recovery accuracy of 98.42\% is achieved at $d = 3$, though it comes with a slight increase in detection rates (51.5\% and 60.9\%). When $d = 10$, PSyDUCK sacrifices some accuracy (94.65\%) in exchange for higher throughput; detection rates increase to 84.2\% and 85.0\%.\\

\paragraph{Qualitative Evaluation}
We display in Figure~\ref{fig:qual-img} several stegoimages generated with SD v2.1 with corresponding embedded messages. As visible, the perceptual difference between cover images and stegoimages is marginal. As a further evidence of the high quality of our generated samples, we illustrate in Figure~\ref{fig:qual-d} the visual effect of varying the divergent step count $d$ on images generated by SD v2.1, with differences from the cover image magnified 20 times. When $d = 1$, the differences are minimal, while more pronounced and visually significant alterations are observed as $d$ increases, particularly at $d = 20$. Based on these findings, in subsequent experiments with latent video diffusion, we focus on divergent step counts of $d = 1$, $2$, and $3$.

\paragraph{Image Perceptual Metrics}

We further evaluated image degradation introduced by the PSyDUCK framework using CLIPScore~\cite{hessel2021clipscore} and Fréchet Inception Distance (FID)~\cite{Seitzer2020FID}. Ideally, a good stegosystem should not introduce significant degradation into cover images. We use CLIPScore to assess semantic alignment between generated images and their textual prompts. FID quantifies perceptual differences between image distributions, hence we evaluate the FID between a collection of sampled stegoimages and their corresponding original covers. We compare using SD v2.1 with varying divergent step counts $d$. The CLIPScore remained relatively stable, with values of 0.8331 on the cover images, 0.8331 with $d=1$, 0.8324 with $d=2$, and 0.8320 with $d=3$, indicating minimal degradation of prompt following. FID measured 0.266 for $d=1$, 3.860 for $d=2$, and 5.217 for $d=3$, showing greater perceptual differences between the cover images and stegoimages as divergence increased.

%3 div steps: FID:  5.217066108006463
%2 div steps: FID:  3.8602121128875524
%1 div step: FID:  0.26624833374648915

%original CLIPscore with sd2.1: CLIPScore: 0.8331
%1 div step: CLIPScore: 0.8331
%2 div step: CLIPScore: 0.8324
%3 div steps: CLIPScore: 0.8320

\subsection{Video Steganography}

In this section, we conduct experiments with the latent video diffusion model SVD. We compare the performance of the PSyDUCK framework with the work of ~\citet{mao2024videostego}, which uses a \textit{trained} encoder-decoder paradigm to encode messages into the latent space of a generated video. 

\begin{table}[t]
    \centering
    \resizebox{\columnwidth}{!}{%
    \begin{tabular}{c c c c c c}
        \toprule
        \multirow{2}{*}{\raisebox{-0.3em}{\textbf{Stegosystem}}} & \multirow{2}{*}{\raisebox{-0.3em}{\shortstack{\textbf{Bytes Encoded} \\ \textbf{Per Frame}}}} & \multirow{2}{*}{\shortstack{\textbf{Transmission} \\ \textbf{Accuracy}}} & \multicolumn{2}{c}{\textbf{Detection Rate}} \\
        \cmidrule(lr){4-5}
        & & & \textbf{SRNet} & \textbf{SiaStegNet} \\
        \midrule
        \citet{mao2024videostego} & 2.25 & \textbf{99.42} & --- & --- \\
        \midrule
        Psyduck ($d = 1$) & 32 & 96.23 & \textbf{50.0} & \textbf{50.0} \\
        Psyduck ($d = 2$) & 96 & 96.21 & 50.0 & 50.7 \\
        Psyduck ($d = 3$) & 96 & 97.95 & 49.8 & 51.0 \\
        \bottomrule
    \end{tabular}%
    }
    \caption{\small \textbf{Results on latent video diffusion experiments.} Detection rates for ~\citet{mao2024videostego} are omitted due to the lack of an open-source implementation or reported results.}
    \label{tab:video-performance}
\end{table}

As shown in Table~\ref{tab:video-performance}, PSyDUCK vastly outperforms the deep steganographic method proposed by~\citet{mao2024videostego}, particularly with respect to encoding capacity. PSyDUCK encodes nearly $14\times$ more information per frame without experiencing significant drops in transmission accuracy. Even at higher divergent step counts ($d = 2$ and $d = 3$), PSyDUCK maintains robust accuracy, with minimal degradation from 96.23\% at $d = 1$ to 97.95\% at $d = 3$. Importantly, PSyDUCK achieves these gains while maintaining low detection rates. In contrast, \citet{mao2024videostego} lacks sufficient open-source results for detection rates, but its relatively low encoding capacity suggests its effectiveness may be limited in higher-throughput settings. This demonstrates PSyDUCK’s ability to balance high encoding throughput and reliable recovery accuracy with minimal susceptibility to detection, even as the divergent step count and encoded data per frame increase.

\begin{figure}[t]
    \setlength{\tabcolsep}{2pt} % default is 6pt, affects column separation
    \renewcommand{\arraystretch}{-0.9} % reduce row spacing in tables
    \centering
    % Adjust width or height as needed, e.g., width=3cm or 0.25\textwidth, etc.
    \resizebox{\linewidth}{!}{%
    \begin{tabular}{ccccc}
         % & $d=1$ & $d=2$ & $d=3$ \\
        
        % Row 0
        % \raisebox{0.8cm}{\rotatebox{90}{Frame 0}} &
        \includegraphics[width=3cm]{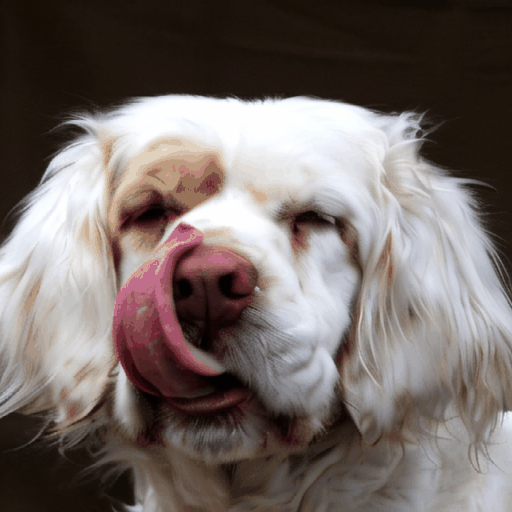} &
        \includegraphics[width=3cm]{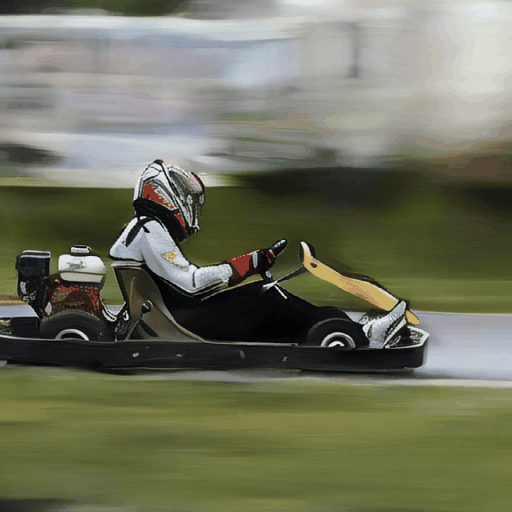} &
        \includegraphics[width=3cm]{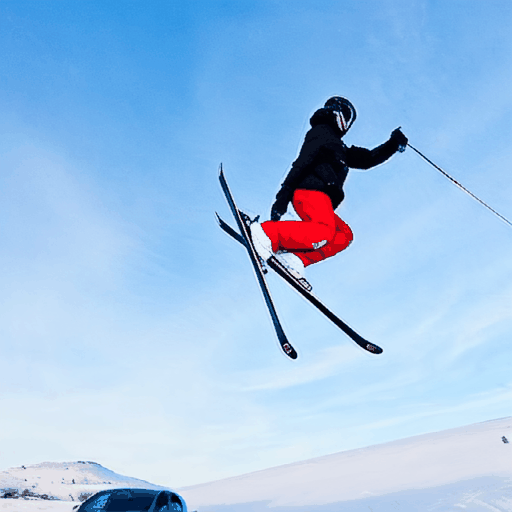} &
        \includegraphics[width=3cm]{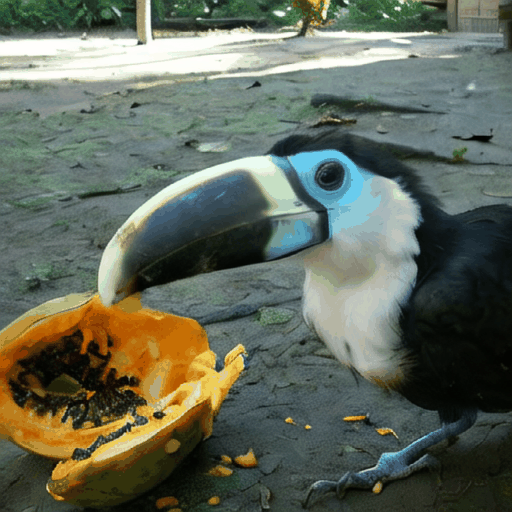} &
        \includegraphics[width=3cm]{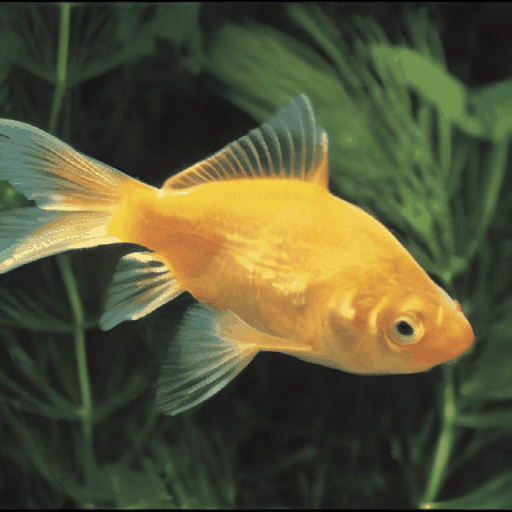}\\[6pt]
        
        % Row 1
        %\raisebox{0.8cm}{\rotatebox{90}{Frame 4}} &
        \includegraphics[width=3cm]{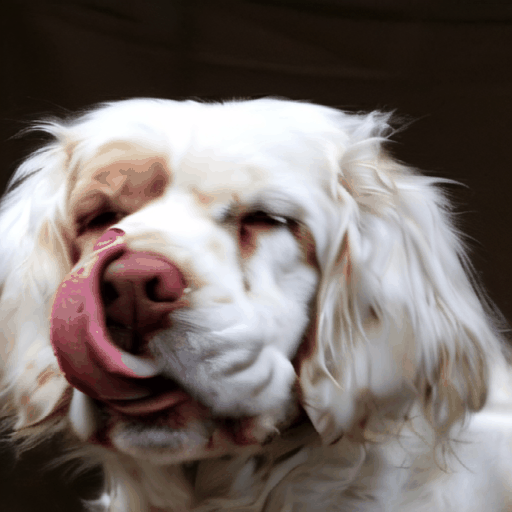} &
        \includegraphics[width=3cm]{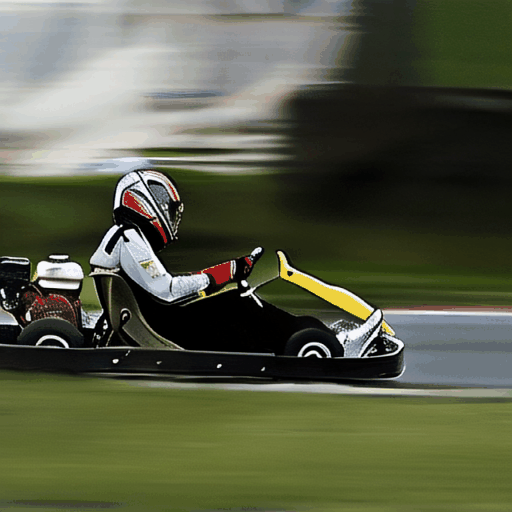} &
        \includegraphics[width=3cm]{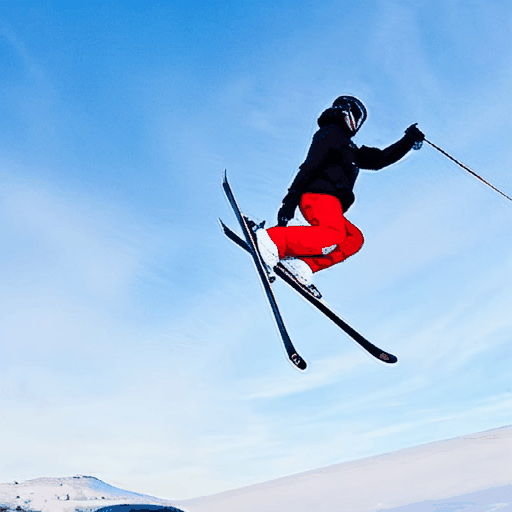} &
        \includegraphics[width=3cm]{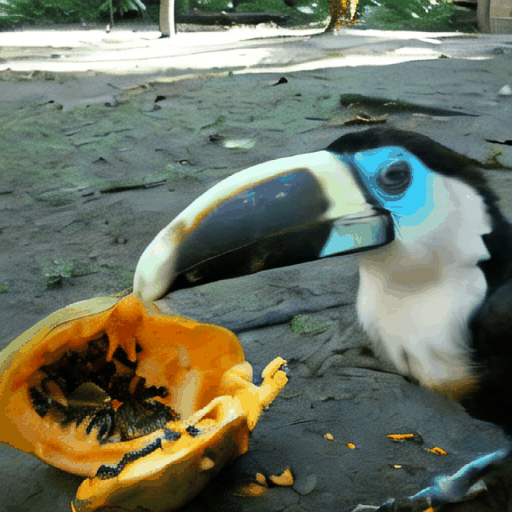} &
        \includegraphics[width=3cm]{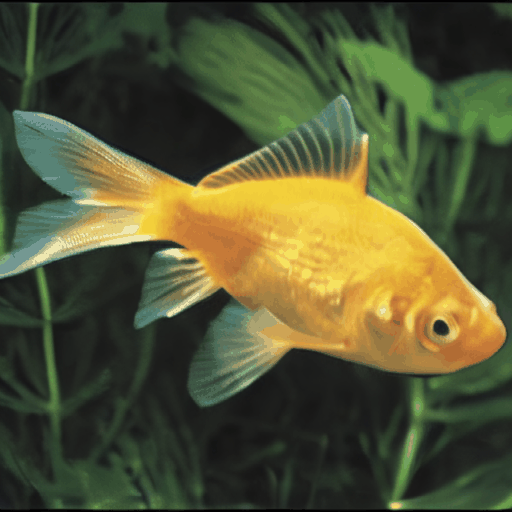}\\[6pt]
        
        % Row 2
        %\raisebox{0.8cm}{\rotatebox{90}{Frame 8}} &
        \includegraphics[width=3cm]{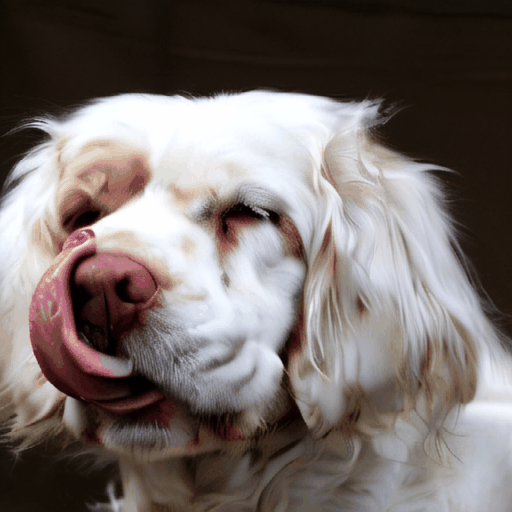} &
        \includegraphics[width=3cm]{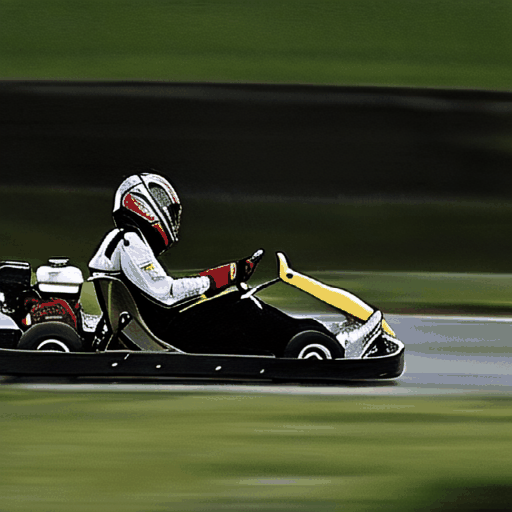} &
        \includegraphics[width=3cm]{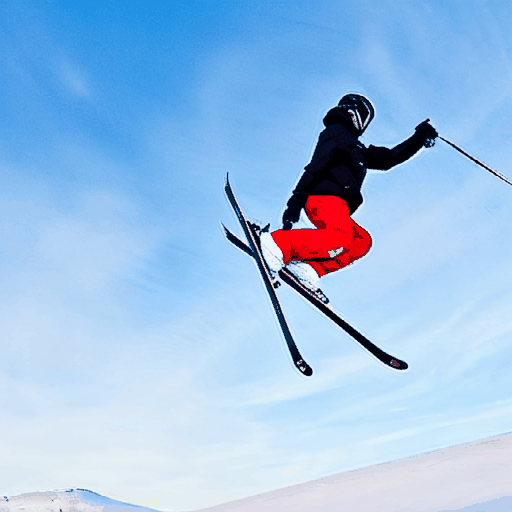} &
        \includegraphics[width=3cm]{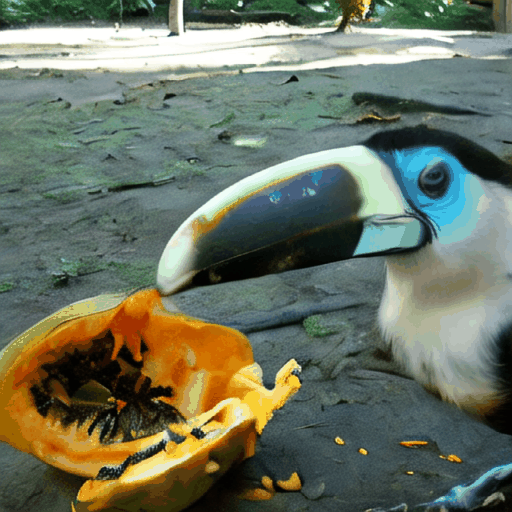} &
        \includegraphics[width=3cm]{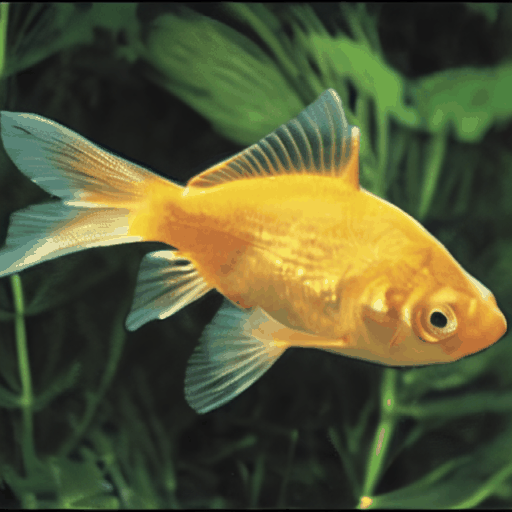}\\[6pt]
        
    \end{tabular}
    }
    \caption{\textbf{Qualitative analysis of SVD stegasamples.} We show three representative frames per video. The visual integrity of the steganographic samples remains high with no noticeable artifacts. The embedded message is content of the Zimmerman Telegraph, reproduced in Appendix~\ref{appendix:experimental-details}.}
    \label{fig:qualitative-video}
\end{figure}

To complement our quantitative results, Figure~\ref{fig:qualitative-video} presents a qualitative analysis of the generated videos, showcasing three representative frames per video. The visual integrity of the steganographic samples remains high, with no visible artifacts that would reveal the presence of hidden information.

\subsection{Ablation Studies}

To provide deeper insights into the behavior of PSyDUCK, we conduct ablation studies on the effects of model precision, the type of base image for generative video diffusion, and various post-processing techniques. 

\paragraph{Effect of Model Precision:} Table~\ref{tab:precision} highlights the impact of reducing model precision from 32-bit floating-point (FP32) to 16-bit floating-point (FP16). This is particularly sigificant in the context of steganographic encoding, since the model precision may impact the ability to recover the message. We use the \texttt{celeb} model for this experiment. From our results, lower precision reduces computational demands but introduces rounding errors that degrade decoding accuracy, particularly as the number of divergence steps $d$ increases. For instance, at $d = 3$, FP32 achieves significantly higher transmission accuracy than FP16, which underscores the importance of precision for tasks requiring higher divergence.

\begin{table}[t]
    \centering
    \resizebox{\columnwidth}{!}{%
    \begin{tabular}{c c c c}
        \toprule
        \textbf{Precision} & \textbf{Divergent Steps $d$} & \textbf{Bytes Encoded} & \textbf{Transmission Accuracy} \\
        \midrule
        \multirow{3}{*}{fp32}
        & $1$ & 512 & 75.00 \\
        & $2$ & 512 & 79.64 \\
        & $3$ & \textbf{512} & \textbf{85.10} \\
        \midrule
        \multirow{3}{*}{fp16}
        & $1$ & 512 & 74.21 \\
        & $2$ & 512 & 75.56 \\
        & $3$ & 512 & 80.47 \\
        \bottomrule
    \end{tabular}}
    \caption{\textbf{Effect of model precision on recovery with} \texttt{celeb}. The table highlights how \texttt{fp32} consistently outperforms \texttt{fp16}. This suggests that higher precision is critical to accurate recovery, particularly in scenarios with more divergent steps.}
    \label{tab:precision}
\end{table}

\begin{table}[t]
    \centering
    \resizebox{\columnwidth}{!}{%
    \begin{tabular}{c c c c}
        \toprule
        \shortstack{\textbf{Base} \\ \textbf{Image Type}} &  \raisebox{0.6em}{\shortstack{\textbf{Stegosystem}}} & \shortstack{\textbf{Bytes Encoded} \\ \textbf{Per Frame}} & \shortstack{\textbf{Transmission} \\ \textbf{Accuracy}} \\
        \midrule
        \multirow{3}{*}{Real} 
        & Psyduck ($d = 1$) & 96 & 96.23 \\
        & Psyduck ($d = 2$) & 96 & 96.21 \\
        & Psyduck ($d = 3$) & \textbf{96} & \textbf{97.95} \\
        \midrule
        \multirow{3}{*}{Synthetic} 
        & Psyduck ($d = 1$) & 96 & 64.78 \\
        & Psyduck ($d = 2$) & 96 & 69.20 \\
        & Psyduck ($d = 3$) & 96 & 72.80 \\
        \bottomrule
    \end{tabular}%
    }
    
    \caption{\textbf{Effect of real or synthetic base image.} We compare the transmission accuracy of latent video diffusion outputs conditioned on either a real or synthetic base image. The results show that using a real base image significantly improves transmission accuracy across all configurations.}
    \label{tab:vid-base-image}
\end{table}

\paragraph{Base Image Type:} Table~\ref{tab:vid-base-image} compares the performance of PSyDUCK when using real versus synthetic images as the base input. When initialized with real images, PSyDUCK consistently achieves significantly higher transmission accuracy across all values of $d$. In contrast, synthetic images, such as those generated by pre-trained diffusion models, result in lower accuracy, possibly due to inherent inconsistencies or variability within their latent representations. This observation suggests that PSyDUCK may be better suited for steganographic applications involving content derived from natural images, where the underlying representations are more stable and coherent. Figure~\ref{fig:weird-svd} further illustrates this issue qualitatively.
%, showing the unusual and visually distorted outputs produced by \texttt{StableVideoDiffusion} when a synthetic image is used as the base.

\section{Discussion}
\label{conclusion}

% We present PSyDUCK, a novel, training-free steganographic framework that leverages controlled divergence and local mixing to enable secure, steganographic message embedding using latent diffusion models. 
% By eschewing the need for retraining or private model sharing, PSyDUCK facilitates practical steganography that cannot be achieved with deep methods.
% Our framework can operate directly on both latent and pixel space architectures, overcoming a flaw of prior methods that were not model-agnostic. 
% Moreover, PSyDUCK allows users to control the strength of their steganographic signal, enabling fine-grained optimization between transmission accuracy and detection likelihood.
% This work demonstrates that performant, training-free steganography can indeed be extended to latent diffusion models and evaluates the performance of PSyDUCK on latent-based generative image and video models. These findings demonstrate PSyDUCK’s effectiveness with fixed-length video models and highlight its potential for scalable, long-duration steganographic applications.
We introduce PSyDUCK, a novel, training-free steganographic framework that leverages controlled divergence for secure message embedding in latent diffusion models. By eliminating the need for retraining or private model sharing, PSyDUCK significantly simplifies practical deployment compared to existing deep-learning methods. Our approach is model-agnostic and dynamically adjusts embedding intensity to balance message accuracy and detection resistance.

We demonstrate PSyDUCK’s capabilities across both latent-based image and video diffusion models, showcasing its suitability for secure, high-capacity steganographic communication. Our experiments confirm PSyDUCK’s effectiveness with fixed-length video, laying the groundwork for future scalable and extended-duration steganographic use cases.

\begin{figure}[t]
    \centering
    \resizebox{\linewidth}{!}{%
    \begin{tabular}{ccccc}
        \includegraphics[]{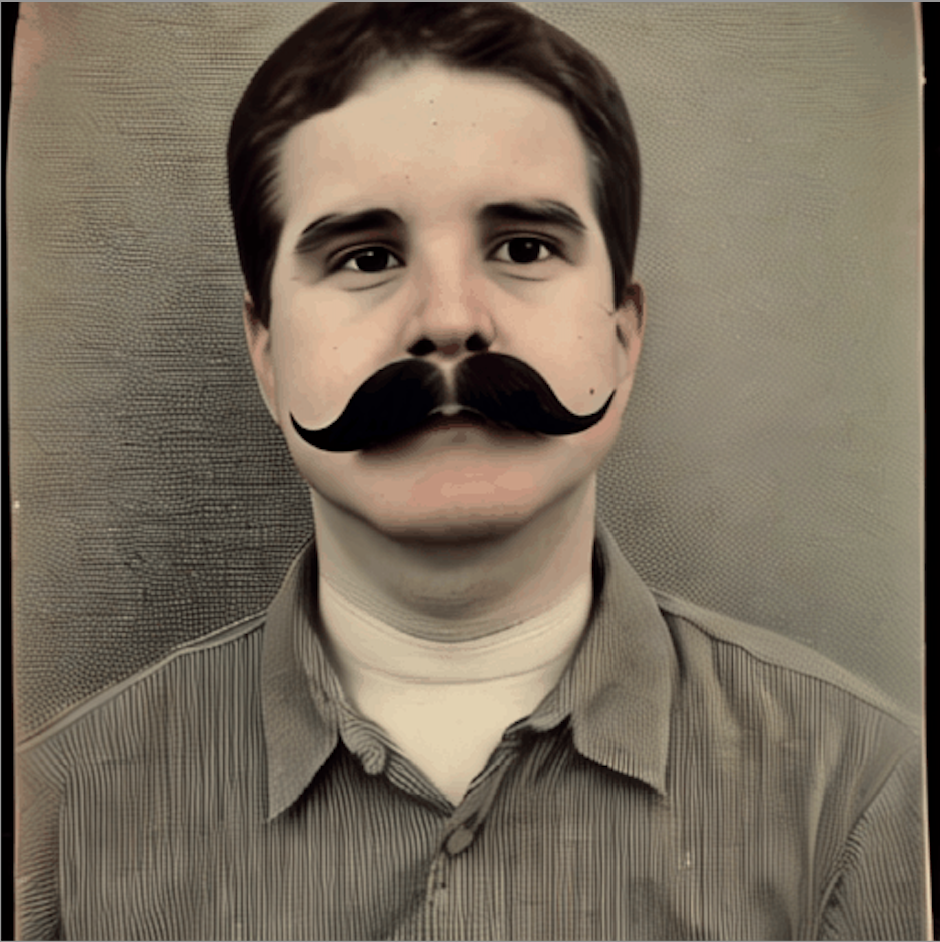} &
        \includegraphics[]{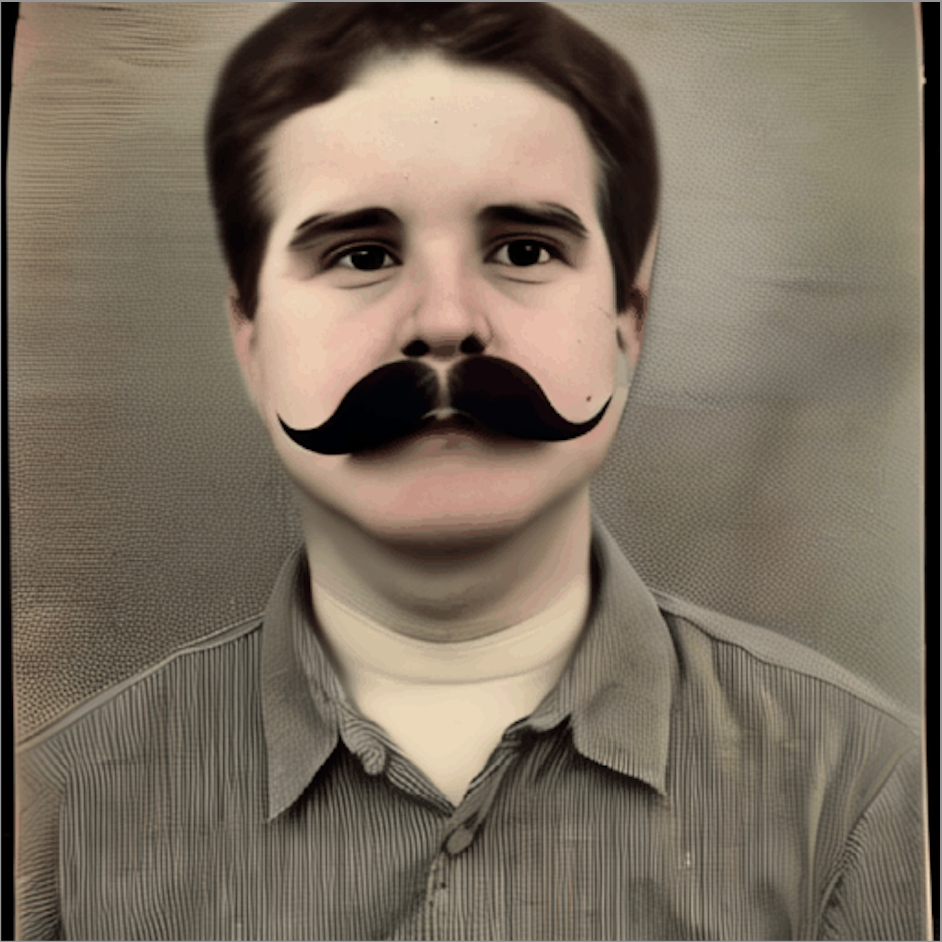} &
        \includegraphics[]{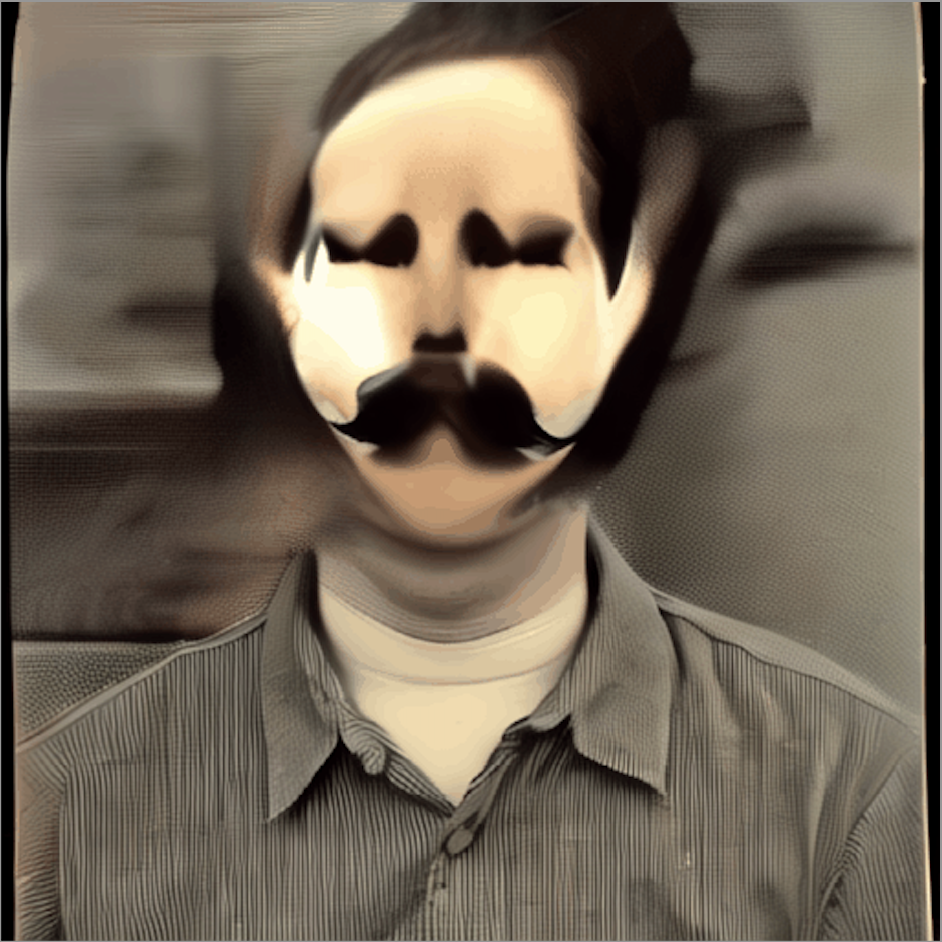} &
        \includegraphics[]{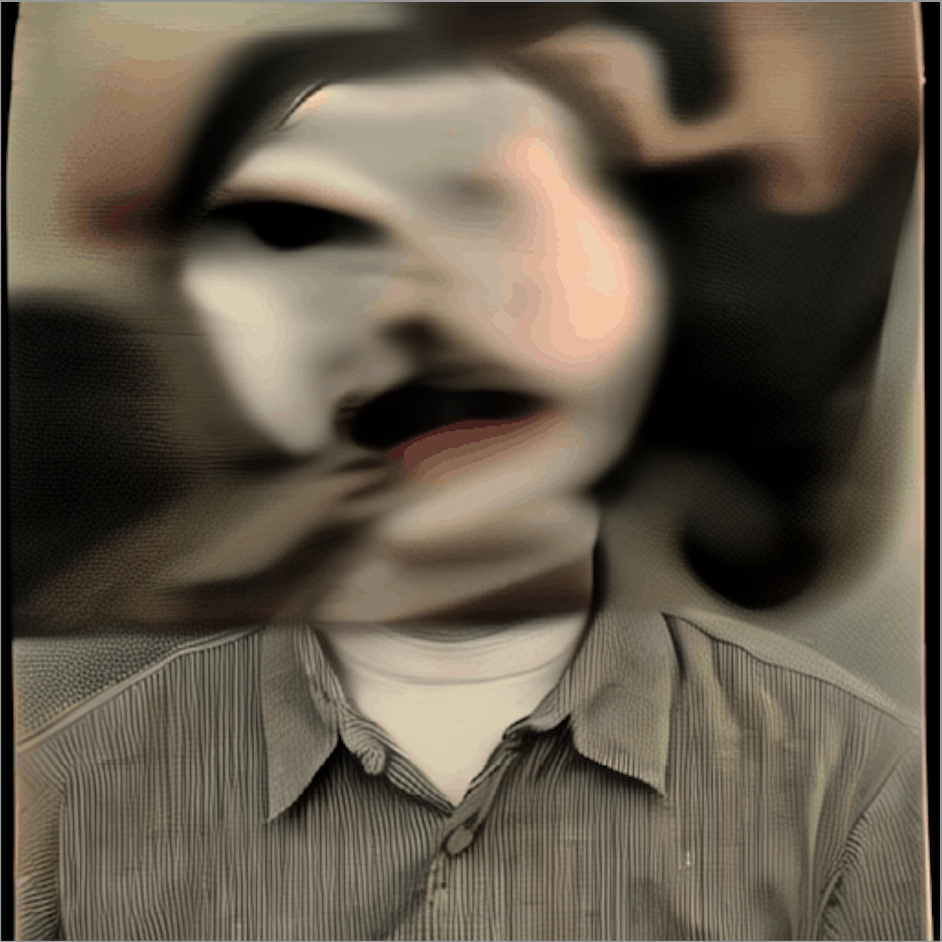} &
        \includegraphics[]{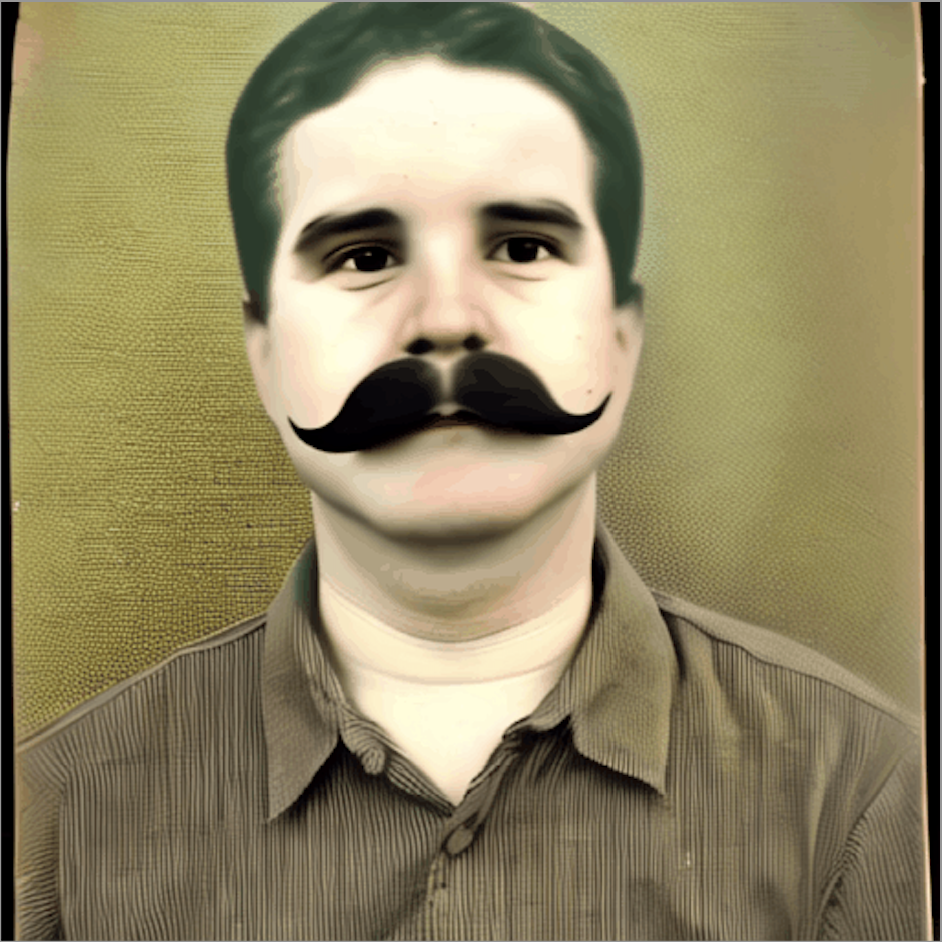}
    \end{tabular}}
    \caption{\textbf{Distorted output of SVD with synthetic base images.} We show five frames of a distorted video output by SVD when we condition on a synthetic base image.}
    \label{fig:weird-svd}
\end{figure}

A clear direction for future work points towards testing PSyDUCK's scalability under more diverse conditions and adversarial scenarios. While our current experiments are limited to fixed-length video models, exploration into FIFO-Diffusion or other unbounded-length video diffusion methods could potentially create steganographic channels of unlimited capacity. We would also like to explore the effect that steganography has on the temporal consistency of the video produced by latent diffusion models, as this could impact the perceptual integrity and security of embedded information over time. Another promising direction involves the development of more robust steganalyzers for video, as current methods for video steganalysis remain underdeveloped. Building on the findings in VANE-Bench~\citep{vane-bench}, we propose that large multimodal foundation models be explored for detecting steganographic anomalies in video content.

\small
\bibliographystyle{ieeenat_fullname}
\bibliography{refs}

\newpage
\appendix
\onecolumn
\section{Proofs}
\label{appendix:proofs}
\setlength{\abovedisplayskip}{5pt} % Space above equations
\setlength{\belowdisplayskip}{5pt} % Space below equations

\begin{lemma}
\label{lem:noise-bounded}
\textbf{(Single Step Bound)}\\
Suppose $\epsilon_\theta$ is bounded as above. Let $\bx_t$ be an arbitrary sample, $t$ a timestep, and $k$ a key. Then, for the update
\begin{equation}
  \bx_{t-1} = \mathrm{DiffusionStep}(\bx_t, t, k),
\end{equation}
there is a constant $\kappa>0$ such that
\begin{equation}
  \|\bx_t - \bx_{t-1}\| \le \kappa \,\sigma_t \quad\text{for all } t.
\end{equation}
Equivalently,
\begin{equation}
   \|\bx_t - \mathrm{DiffusionStep}(\bx_t, t, k)\| \in O(\sigma_t).
\end{equation}
\end{lemma}

\begin{proof}
By assumption, the noise estimator $\epsilon_\theta(\bx_t, t)$ remains within a bounded norm. In many diffusion sampling schemes (e.g., DDPM, DDIM), the update from $\bx_t$ to $\bx_{t-1}$ is of the form
\begin{equation}
  \bx_{t-1} = \bx_t + f(\bx_t, t, k) + g(t)\,\epsilon_\theta(\bx_t, t),
\end{equation}
where $g(t)$ is typically proportional to $\sigma_t$. Because $\|\epsilon_\theta(\bx_t, t)\|\le C$ for some constant $C$, one can choose a suitable constant $\kappa$ such that $\|f(\bx_t,t,k)\| + |g(t)|\,\|\epsilon_\theta(\bx_t,t)\| \le \kappa\,\sigma_t$ for all $\bx_t$. Hence,
\begin{equation}
  \|\bx_t - \bx_{t-1}\| = \bigl\|f(\bx_t, t, k) + g(t)\,\epsilon_\theta(\bx_t,t)\bigr\| \le \kappa \,\sigma_t.
\end{equation}
\end{proof}

\begin{lemma}
\label{lem:dstep-bounded}
\textbf{($d$-Step Bound)}\\
Fix a key $k$. Let $\bx_{t}$ be a sample at timestep $t$, and let $\bx_{t-d}$ be the sample obtained by applying $d$ iterations of $\mathrm{DiffusionStep}(\,\cdot\,, k)$. Under the same assumptions as Lemma~\ref{lem:noise-bounded}, we have
\begin{equation}
  \|\bx_t - \bx_{t-d}\| \in O\bigl(d\, \sigma_t\bigr).
\end{equation}
\end{lemma}

\begin{proof}
We apply the telescoping sum:
\begin{equation}
  \|\bx_t - \bx_{t-d}\| \le \sum_{i=0}^{d-1}\|\bx_{t-i} - \bx_{t-i-1}\|.
\end{equation}
By Lemma~\ref{lem:noise-bounded}, each term satisfies $\|\bx_{t-i} - \bx_{t-i-1}\|\le \kappa\,\sigma_{t-i}$.  If $\sigma_{t-i}$ is comparable to $\sigma_t$ for $i=0,\dots,d-1$ (i.e., the noise schedule does not change drastically), then there is a constant $\kappa'$ such that $\sigma_{t-i}\le \kappa' \,\sigma_t$. Hence,
\begin{equation}
  \|\bx_t - \bx_{t-d}\| \le \sum_{i=0}^{d-1}\kappa\,\sigma_{t-i} \le \kappa\,\kappa' \sum_{i=0}^{d-1}\sigma_t = \kappa\,\kappa'\, d\,\sigma_t \in O\bigl(d\,\sigma_t\bigr).
\end{equation}
\end{proof}

\begin{lemma}
\label{lem:mix-bounded}
Consider the local function
\begin{equation}
  {\sf Mix}(\bb, X_1) := \text{(elementwise) selecting from a list of diverged samples } X_1.
\end{equation}
Let $\bx_{d+1}$ be the sample at time $d+1$ just prior to divergence. If each ``diverged path'' differs from $\bx_{d+1}$ by at most $O(d\,\sigma_{d+1})$, then
\begin{equation}
  \| {\sf Mix}(\bb, X_1) - \bx_{d+1}\| \in O(r\, d\, \sigma_{d+1}),
\end{equation}
where $r$ is the number of diverged samples being mixed.
\end{lemma}

\begin{proof}
Let $X_1=[\bx^0_{1},\ldots,\bx^{r-1}_1]$ be the diverged samples at time $1$. For each individual coordinate $j$, ${\sf Mix}$ picks exactly one of $\{x^0_{1,j},\dots,x^{r-1}_{1,j}\}$. By nonnegativity of norms, we have
\begin{equation}
  \|{\sf Mix}(\bb,X_1)-\bx_{d+1}\| = \sum_{j=1}^{\ell}\bigl\| x^b_{1,j} - x_{d+1,j}\bigr\| \le \sum_{j=1}^{\ell}\sum_{i=0}^{r-1}\bigl\| x^i_{1,j} - x_{d+1,j}\bigr\|.
\end{equation}
Rearranging terms,
\begin{equation}
  \sum_{j=1}^{\ell}\sum_{i=0}^{r-1}\bigl\| x^i_{1,j} - x_{d+1,j}\bigr\| = \sum_{i=0}^{r-1} \sum_{j=1}^{\ell}\bigl\| x^i_{1,j} - x_{d+1,j}\bigr\| = \sum_{i=0}^{r-1}\|\bx^i_{1} - \bx_{d+1}\|.
\end{equation}
If each diverged sample $\bx^i_1$ is within $O(d\,\sigma_{d+1})$ of $\bx_{d+1}$ (as in Lemma~\ref{lem:dstep-bounded}), then summing over $i=0,\dots,r-1$ yields
\begin{equation}
  \sum_{i=0}^{r-1} O\bigl(d\,\sigma_{d+1}\bigr) = O\bigl(r\,d\,\sigma_{d+1}\bigr).
\end{equation}
Hence, the entire norm difference is $O(r\,d\,\sigma_{d+1})$.
\end{proof}

\paragraph{Extended Proof of Proposition~\ref{prop:err-sec-bound-formal}}

We now combine Lemmas~\ref{lem:noise-bounded}--\ref{lem:mix-bounded} to show $\|\alice{\bx_0} - \cover{\bx_0}\| \in O(d\,r\,\sigma_{d+1})$.

\begin{proof}[Proof of Proposition~\ref{prop:err-sec-bound-formal}]
Recall that $\alice{\bx_0}$ differs from $\cover{\bx_0}$ in three main ways:

\begin{enumerate}
    \item From $\alice{\bx_0}$ to $\mathsf{Mix}(\bb,X_1)$. By an argument analogous to Lemma~\ref{lem:noise-bounded}, we have
    \begin{equation}
      \|\alice{\bx_0} - \mathsf{Mix}(\bb,X_1)\| \le O\bigl(\sigma_1\bigr).
    \end{equation}

    \item From $\mathsf{Mix}(\bb,X_1)$ to $\alice{\bx_{d+1}}$. By Lemma~\ref{lem:mix-bounded}, the difference between $\mathsf{Mix}(\bb,X_1)$ and $\bx_{d+1}$ is $O(r\,d\,\sigma_{d+1})$. Hence,
    \begin{equation}
      \|\mathsf{Mix}(\bb,X_1) - \alice{\bx_{d+1}}\| \le O\bigl(r\,d\,\sigma_{d+1}\bigr).
    \end{equation}

    \item From $\alice{\bx_{d+1}}$ to $\cover{\bx_0}$. By Lemma~\ref{lem:dstep-bounded} (or a similar argument going from $t=d+1$ to $t=0$),
    \begin{equation}
      \|\alice{\bx_{d+1}} - \cover{\bx_0}\| \le O\bigl(d\,\sigma_{d+1}\bigr).
    \end{equation}
\end{enumerate}

Using the triangle inequality on the sum of these differences,
\begin{equation}
  \|\alice{\bx_0} - \cover{\bx_0}\| \le \|\alice{\bx_0} - \mathsf{Mix}(\bb,X_1)\| + \|\mathsf{Mix}(\bb,X_1) - \alice{\bx_{d+1}}\| + \|\alice{\bx_{d+1}} - \cover{\bx_0}\|.
\end{equation}
Hence,
\begin{equation}
  \|\alice{\bx_0} - \cover{\bx_0}\| \in O\bigl(\sigma_1\bigr) + O\bigl(r\,d\,\sigma_{d+1}\bigr) + O\bigl(d\,\sigma_{d+1}\bigr) \in O\bigl(d\,r\,\sigma_{d+1}\bigr).
\end{equation}
For large $d$, we typically consider $\sigma_1$ small or absorbed into a constant factor, so the final complexity is $O(d\,r\,\sigma_{d+1})$.
\end{proof}

\paragraph{Extended Proof of Proposition~\ref{prop:sto-d1-secure-formal}}

\begin{proof}
Suppose a ``cover'' sample is generated by standard diffusion:
\begin{equation}
  \cover{\bx_0} = \mathrm{DiffusionStep}(\bx_1,1,k_s) = \mathrm{Model}(\bx_1) + \sigma_1\,\epsilon.
\end{equation}
Going one step further, $\bx_1$ itself could be written similarly as $\bx_1 = \mathrm{Model}(\bx_2) + \sigma_2\,\tilde{\epsilon}$, etc.

Under Psyduck with $d=1$, we introduce a local mixture of noise $\mathsf{Mix}(\bb,\{\epsilon_0,\dots,\epsilon_{r-1}\})$. The sample at time $t=1$ is
\begin{equation}
  \alice{\bx_1} = \mathsf{Mix}(\bb,\alice{X_1}),
\end{equation}
where $\alice{X_1} = [\bx^0_1,\ldots,\bx^{r-1}_1]$, each $\bx^i_1$ being of the form $\mathrm{Model}(\bx_2) + \sigma_2\,\epsilon_i$. By linearity of $\mathsf{Mix}$,
\begin{equation}
  \mathsf{Mix}(\bb,\alice{X_1}) = \mathrm{Model}(\bx_2) + \sigma_2 \,\mathsf{Mix}\bigl(\bb,[\epsilon_0,\dots,\epsilon_{r-1}]\bigr).
\end{equation}
Thus,
\begin{equation}
  \alice{\bx_0} = \mathrm{DiffusionStep}(\alice{\bx_1},1,k_s) = \mathrm{Model}\bigl(\alice{\bx_1}\bigr) + \sigma_1\,\epsilon.
\end{equation}
But here, $\alice{\bx_1}$ contains the mixture of noise.

Compare:
\begin{equation}
  \cover{\bx_0} = \mathrm{Model}(\bx_1) + \sigma_1\,\epsilon, \quad \alice{\bx_0} = \mathrm{Model}\bigl(\mathsf{Mix}(\bb,\alice{X_1})\bigr) + \sigma_1\,\epsilon.
\end{equation}
As long as $\mathrm{Model}$ is a deterministic function and $\sigma_1\,\epsilon$ appears as pure noise, the adversary's only handle is whether $\mathsf{Mix}(\bb,\{\epsilon_0,\dots,\epsilon_{r-1}\})$ is distinguishable from fresh noise $\epsilon$. If these distributions are indistinguishable (e.g., all are i.i.d.\ Gaussian and $\mathsf{Mix}$ preserves randomness in a black-box manner), then the adversary gains no information about the presence of steganography. Consequently, the scheme is secure under the assumption of indistinguishable noise.
\end{proof}

\section{Algorithms} \label{appendix:algorithms}

% \vspace{-15pt}
{
\newcommand{\highlight}[4]{
    \tikz[overlay, remember picture]{
        \fill[#2!50, opacity=0.3] 
        ([xshift={#3[0]}, yshift={#3[1]}]pic cs:#1-start) rectangle ([xshift={#4[0]}, yshift={#4[1]}]pic cs:#1-end);
    }
}
\newcommand{\psyducklabel}[2]{
    {\pgfsetfillopacity{0.15}\colorbox{#1}{\pgfsetfillopacity{1}\textcolor{black}{#2}}\pgfsetfillopacity{1}}
}

\begin{figure}[t]
\begin{minipage}[t]{0.48\textwidth}
\begin{algorithm}[H]
    \caption{PSyDUCK Encode}
    \label{alg:psyduck_encode}
    \begin{algorithmic}[1]
        \REQUIRE $k_s, \{k_{i \in 0, ..., r}\}, d, \bb$
        \tikzmark{enc1-start}
        \STATE $\bx_T \gets {\sf Preprocess}(k_s)$
        \STATE $t \gets T$
        \tikzmark{enc1-end}
        %\highlight{red}{enc1}{-11em,-0.3em}{13em,-0.3em}

        \tikzmark{enc2-start}
        \WHILE{$t > d+1$} 
            \STATE $\bx_{t-1} \gets {\sf DiffusionStep}(\bx_t, t, k_s)$
            \STATE $t \gets t - 1$
        \ENDWHILE
        \tikzmark{enc2-end}
        % \highlight{orange}{enc2}{-3.5em,-0.3em}{11.25em,-0.3em}

        \tikzmark{enc3-start}
        \FOR{$i = 0, ..., r-1$}
            \STATE $\bx^i_{t} \gets \bx_t$
        \ENDFOR
        \tikzmark{enc3-end}
        %\highlight{yellow}{enc3}{-5.25em,-0.3em}{12.5em,-0.3em}

        \tikzmark{enc4-start}
        \WHILE{$t > 1$} 
            \FOR{$i = 0, ..., r-1$}
                \STATE $\bx^i_{t-1} \gets {\sf DiffusionStep}(\bx^i_t, t, k_i)$
            \ENDFOR
            \STATE $t \gets t - 1$
        \ENDWHILE
        \tikzmark{enc4-end}
        %\highlight{green}{enc4}{-4.0em,-0.3em}{11.25em,-0.3em}

        \STATE $\bx_1 \gets {\sf Mix}\bigl(\bb, [\bx_1^0, ..., \bx_1^{r-1}]\bigr)$ 
        \STATE $\alice{\bx_0} \gets {\sf DiffusionStep}(\bx_1, 1, k_s)$
        \STATE $\alice{\bx_0} \gets {\sf Postprocess}(\alice{\bx_0}, k_s)$
        % \RETURN $\alice{\bx_0}$ 
    \end{algorithmic}
\end{algorithm}
\end{minipage}
\hfill
\begin{minipage}[t]{0.48\textwidth}
\begin{algorithm}[H]
    \caption{PSyDUCK Decode}
    \label{alg:psyduck_decode}
    \begin{algorithmic}[1]
        \REQUIRE $k_s, \{k_{i \in 0, ..., r}\}, d, \alice{\bx_0}$
        \tikzmark{dec1-start}
        \STATE $\bx_T \gets {\sf Preprocess}(k_s)$
        \STATE $t \gets T$
        \tikzmark{dec1-end}
        %\highlight{red}{dec1}{-12em,-0.3em}{13em,-0.3em}

        \tikzmark{dec2-start}
        \WHILE{$t > d+1$} 
            \STATE $\bx_{t-1} \gets {\sf DiffusionStep}(\bx_t, t, k_s)$
            \STATE $t \gets t - 1$
        \ENDWHILE
        \tikzmark{dec2-end}
        %\highlight{orange}{dec2}{-3.5em,-0.3em}{11.25em,-0.3em}

        \tikzmark{dec3-start}
        \FOR{$i = 0, ..., r-1$} 
            \STATE $\bx^i_{t} \gets \bx_{t}$
        \ENDFOR
        \tikzmark{dec3-end}
        %\highlight{yellow}{dec3}{-5.25em,-0.3em}{12.5em,-0.3em}

        \tikzmark{dec4-start}
        \WHILE{$t > 0$} 
            \FOR{$i = 0, ..., r-1$}
                \STATE $\bx^i_{t-1} \gets {\sf DiffusionStep}(\bx^i_t, t, k_i)$
            \ENDFOR
            \STATE $t \gets t - 1$
        \ENDWHILE
        \tikzmark{dec4-end}
        %\highlight{cyan}{dec4}{-4.0em,-0.3em}{11.25em,-0.3em}

        \FOR{$i = 0, ..., r$} 
            \STATE $\mathbf{d}_i \gets |\bx_0^i - \alice{\bx_0}|$
        \ENDFOR
        \STATE $\bbh \gets \bzero^{\ell}$
        
        % \tikzmark{dec5-start}
        \FOR{$j = 1, ..., \ell$} 
            \STATE $\bbh[j] \gets \argmin_{i}\{\mathbf{d}_i[j]\}$
        \ENDFOR
        % \tikzmark{dec5-end}
        % \RETURN $\bbh$
    \end{algorithmic}
\end{algorithm}
\end{minipage}

\highlight{enc1}{red}{{-12.25em,-0.3em}}{{12.25em,-0.3em}}
\highlight{dec1}{red}{{-13.5em,-0.3em}}{{12.5em,-0.3em}}

\highlight{enc2}{orange}{{-1.75em,-0.3em}}{{11.25em,-0.3em}}
\highlight{dec2}{orange}{{-1.75em,-0.3em}}{{11.25em,-0.3em}}

\highlight{enc3}{yellow}{{-1.75em,-0.3em}}{{12.5em,-0.3em}}
\highlight{dec3}{yellow}{{-1.75em,-0.3em}}{{12.5em,-0.3em}}

\highlight{enc4}{green}{{-1.75em,-0.3em}}{{11.25em,-0.3em}}
\highlight{dec4}{cyan}{{-1.75em,-0.3em}}{{11.25em,-0.3em}}

\highlight{dec5}{blue}{{-1.75em,-0.3em}}{{12.5em,-0.3em}}

\vspace{-45pt}
\psyducklabel{red}{Preprocess (i.e. VAE encoding, initial noise generation, etc.) using $k_s$.} \\
\psyducklabel{orange}{Denoise along original trajectory using $k_s$.} \\
\psyducklabel{yellow}{Make $r$ copies of $\bx_{d+1}$ to use for diverging.} \\
\psyducklabel{green}{Alice diverges until final denoising step. \hspace{1.65cm}(i.e. for $t = d, d-1, ..., 1$)} \\
\psyducklabel{cyan}{Bob diverges until all samples are fully denoised. (i.e. for $t = d, d-1, ..., 0$)} \\
% \psyducklabel{blue}{Bob guesses $\bbh$ with local estimation.}

\end{figure}
}
    
\pagebreak
\section{Experimental Details}
\label{appendix:experimental-details}

\subsection{Prompts \& Messages}

We use the following prompts to generate the latent image diffusion images:
\begin{itemize}
    \item A man with a mustache.
    \item A photo of a cat.
    \item Eiffel Tower under the blue sky.
    \item Sydney Opera House.
    \item Leaning Tower of Pisa.
    \item Young girl with blond hair.
    \item A cute rabbit.
    \item Tomatoes hanging on a tree.
    \item A multi-layered sandwich.
    \item A futuristic cityscape at dusk, with neon lights reflecting on wet streets, flying cars, and towering skyscrapers.
    \item A serene fantasy forest with glowing plants, a crystal-clear stream, and mythical creatures like fairies and unicorns.
    \item A steampunk airship soaring above a cloud-covered landscape, with intricate gears and brass details, in a sunset sky.
    \item An astronaut exploring an alien planet with colorful rock formations, strange plants, and distant galaxies visible in the sky.
    \item A whimsical underwater scene with a variety of vibrant marine life, coral reefs, and an ancient sunken ship.
    \item A cozy, rustic cabin in the middle of a snowy mountain landscape, with warm light glowing from the windows and smoke rising from the chimney.
    \item A bustling medieval marketplace with stalls selling exotic goods, people in period clothing, and a grand castle in the background.
    \item A surreal dreamscape with floating islands, cascading waterfalls, and a giant moon hanging low in the sky.
    \item A detailed cyberpunk street scene with diverse characters, holographic advertisements, and high-tech gadgets amidst a gritty urban environment.
    \item A magical library with towering bookshelves, enchanted floating books, and a grand chandelier, all set in a grand, gothic-style room.
\end{itemize}

For experiments with \texttt{StableVideoDiffusion}, we condition the video generation on an input image. We randomly sampled images from ImageNet for all our video tests, cropping them to $512\times512$ when necessary~\cite{imagenet}. 

The following is the text of the Zimmerman Telegram, which we embed in Figure~\ref{fig:qualitative-video}:

\begin{displayquote}
We intend to begin on the first of February unrestricted submarine warfare. We shall endeavor in spite of this to keep the United States of America neutral. In the event of this not succeeding, we make Mexico a proposal of alliance on the following basis: make war together, make peace together, generous financial support, and an understanding on our part that Mexico is to reconquer the lost territory in Texas, New Mexico, and Arizona. The settlement in detail is left to you. You will inform the President of the above most secretly as soon as the outbreak of war with the United States of America is certain and add the suggestion that he should, on his own initiative, invite Japan to immediate adherence and at the same time mediate between Japan and ourselves. Please call the President's attention to the fact that the ruthless employment of our submarines now offers the prospect of compelling England in a few months to make peace.
\end{displayquote}

\end{document}